\documentclass[12pt]{article}

\usepackage{amsmath,amssymb,amsthm,mathtools,bbm}
\usepackage{geometry}
\usepackage{enumitem}
\usepackage{setspace}
\usepackage{float}
\usepackage{threeparttable}
\usepackage{booktabs}
\usepackage{graphicx}
\usepackage{float} 
\usepackage{natbib}

\usepackage{algorithm}
\usepackage{algpseudocode}
\usepackage{float}  

\title{
Optimal Policy Learning \\
under Budget and Coverage Constraints
\thanks{
I would like to express my sincere gratitude to Charles Manski, Fabrizia Mealli, Roy Cerqueti, and Franco Peracchi for their valuable comments and insightful suggestions. Needless to say, any remaining errors or omissions remain the sole responsibility of the author.
}
\thanks{
The development of this paper was supported by FOSSR (Fostering Open Science in Social Science Research), a project funded by the European Union -- NextGenerationEU under the PNRR Grant Agreement No. MURIR0000008.
}
}

\author{Giovanni Cerulli \\ CNR-IRCrES \\ National Research Council of Italy \\ Research Institute on Sustainable Economic Growth}
\date{}

\newtheorem{assumption}{Assumption}
\newtheorem{definition}{Definition}
\newtheorem{theorem}{Theorem}
\newtheorem{proposition}{Proposition}
\newtheorem{corollary}{Corollary}
\newtheorem{lemma}{Lemma}
\newtheorem{remark}{Remark}
\usepackage{amsthm}

\begin{document}
\maketitle

\begin{abstract}
\noindent
\footnotesize{We study optimal policy learning under combined budget and minimum coverage constraints. We show that the problem admits a knapsack-type structure and that the optimal policy can be characterized by an affine threshold rule involving both budget and coverage shadow prices. 
We establish that the linear programming relaxation of the combinatorial solution has an $O(1)$ integrality gap, implying asymptotic equivalence with the optimal discrete allocation. Building on this result, we analyze two implementable approaches: a Greedy--Lagrangian (GLC) and a rank-and-cut (RC) algorithm.
We show that the GLC closely approximates the optimal solution and achieves near-optimal performance in finite samples. By contrast, RC is approximately optimal whenever the coverage constraint is slack or costs are homogeneous, while misallocation arises only when cost heterogeneity interacts with a binding coverage constraint. Monte Carlo evidence supports these findings.}
\end{abstract}

\section{Introduction}

Decision makers often face the problem of allocating limited resources across a given set of units. Examples include assigning training programs, medical treatments, or social benefits. In these settings, individuals may differ both in their expected gains from treatment and in the cost of being treated. At the same time, policymakers face realistic constraints as, for example, a limited budget to spend, as well as a minimum coverage requirement to ensure that a sufficient fraction of the population is reached by the intervention.

\bigskip
\noindent
To be concrete, consider a public employment agency that must decide how to allocate a training program among unemployed individuals. Some individuals are expected to benefit substantially from training; for instance, younger workers with some prior education may gain more, while others may benefit less. At the same time, the cost of providing training is not the same for everyone, as individuals with lower initial skills or more fragile backgrounds may require longer and more intensive, and therefore more costly, programs.

\bigskip
\noindent
The agency cannot simply treat everyone, because it operates under a limited budget. It must therefore decide whom to prioritize. However, policy objectives are not purely efficiency-driven. In many real-world settings, there is also a requirement to reach a sufficiently large share of the population to ensure, for example, fairness, political acceptability, or compliance with program guidelines. This means that the agency cannot focus only on a small group of high-return individuals but must also include a broader segment of the population.

\bigskip
\noindent
In this situation, the decision problem becomes non-trivial. Targeting only those with the highest expected gains may violate the coverage requirement, while expanding treatment to meet coverage may force the agency to include individuals with lower expected benefits or higher costs. The policymaker must therefore balance three elements simultaneously: (i) how much individuals are expected to benefit, (ii) how costly they are to treat, and (iii) how many people must be reached.

\bigskip
\noindent
This paper studies optimal policy learning under a budgetary and a coverage constraint. The policy maker must decide which units to treat in order to maximize total gains, while respecting both a budget and a minimum coverage requirement. Although the problem is conceptually simple, in real-world settings it gives rise to a \textit{combinatorial optimization} problem, as the policy maker must select the best subset of units from an exponentially large set of alternatives. In particular, when treatment must be allocated among $N$ units, the number of possible allocations is $2^{N}$, which grows exponentially with the number of units.

\bigskip
\noindent
In this context, the optimal solution can be in principle computed using \textit{mixed-integer linear programming} (MILP). However, such methods do not scale well with the sample size and become impractical in large-scale applications, when the number of units is huge. Unlike linear programming (LP), which can be solved in polynomial time, MILP is NP-hard (\cite{NemhauserWolsey1988}; \cite{Wolsey1998}; \citet{Schrijver2003}; \citet{Vielma2015}). This raises a central issue of how to approximate the optimal policy using simple and scalable rules, without solving the full combinatorial problem.

\bigskip
\noindent
The main contribution of this paper is to show that simpler solvers are feasible. To this end, we first show that the constrained policy learning problem has a \textit{knapsack-type} problem structure and admits a simple threshold-based characterization. We then establish that the linear programming (LP) relaxation of the knapsack problem has an $O(1)$ integrality gap, that is, the gap between the relaxed and the optimal discrete solution vanishes in per-capita terms as the sample size grows. This result provides a theoretical foundation for the use of simpler allocation rules (i.e., \textit{greedy} algorithms).

\bigskip
\noindent
Building on this result, we study two \textit{greedy} implementable selection algorithms. The first is a \textit{Greedy-Lagrangian} rule (GLC), which assigns individuals based on a score that balances treatment effects and costs. The second is a \textit{rank-and-cut} rule (RC), which ranks individuals according to their benefit-cost ratio and constructs a feasible allocation through a simple threshold-based selection procedure. Both rules are computationally easy to apply, and scale non-exponentially with the sample size. We show that the GLC closely approximates the optimal allocation, while the RC provides an even simpler and intuitive alternative based on cost-effectiveness ranking. Despite its simplicity, the RC rule can however produce relevant misallocations when costs are highly heterogeneous and the coverage requirement is binding. Importantly, both rules avoid solving the underlying combinatorial optimization problem.

\bigskip
\noindent
We complement the theoretical analysis with two separate Monte Carlo studies, one for GLC and one for RC. The simulations confirm the main predictions of the theory. The integrality gap of the LP relaxation rapidly vanishes as the sample size increases, providing a tight benchmark for comparison. The results for GLC show that it consistently delivers near-optimal allocations and exhibits stable performance even in finite samples. The analysis of RC, on the contrary, highlights that misallocations emerge when cost heterogeneity is substantial and the coverage constraint is binding, in line with the theoretical predictions.

\medskip
\noindent
The paper makes three main contributions. First, it provides a structural representation of constrained policy learning as a \textit{knapsack-type} problem with a threshold-rule characterization. Second, it establishes that the LP relaxation is asymptotically tight in per-capita terms. Third, it proposes and analyzes simple allocation rules that are near-optimal and scalable, under mild and verifiable conditions.

\medskip
\noindent
The remainder of the paper is organized as follows. 
Section 2 introduces the policy learning problem under budget and minimum coverage constraints and presents its population-level formulation. 
Section 3 provides the primal--dual characterization of the optimal policy and derives the associated threshold rule. 
Section 4 studies the empirical counterpart of the problem and shows that it can be formulated as a constrained knapsack problem. 
Section 5 analyzes the linear programming relaxation and establishes the $O(1)$ integrality gap and asymptotic equivalence with the discrete solution. 
Section 6 introduces the Greedy--Lagrangian (GLC) algorithm and studies its computational and theoretical properties. 
Section 7 characterizes the rank-and-cut (RC) rule and derives conditions under which it approximates the optimal policy. 
Section 8 presents Monte Carlo evidence illustrating the theoretical results. 
Section 9 concludes. 
Additional technical details and numerical illustrations are provided in the Appendix.

\section{Related Literature}

This paper relates to several strands of literature at the intersection of causal inference, empirical welfare maximization, and constrained decision-making.

\bigskip
\noindent
A natural starting point is the literature on statistical treatment rules, which frames policy assignment as an optimization problem under uncertainty. \citet{Manski2004} provides a foundational formulation in which treatment rules are chosen to maximize expected welfare when the data-generating process is unknown. Building on this approach, \citet{KitagawaTetenov2018} introduce empirical welfare maximization (EWM), which selects policies by maximizing sample analogues of welfare within a restricted policy class. \citet{AtheyWager2021} extend this framework to observational data using doubly robust methods and provide regret guarantees, while \citet{MbakopTabordMeehan2021} further generalize the approach through penalized welfare maximization. These contributions establish a statistical foundation for data-driven policy learning.

\bigskip
\noindent
A line of research strictly related to our paper focuses on treatment assignment under resource constraints. \citet{BhattacharyaDupas2012} study optimal treatment assignment under a binding budget constraint, showing that the welfare-maximizing policy takes the form of a threshold rule based on individual treatment effects (or cost-adjusted gains when costs are heterogeneous). They characterize the optimal allocation at the population level and develop inference procedures for welfare and targeting rules using experimental data. \citet{MbakopTabordMeehan2021} extend the empirical welfare maximization framework by introducing a penalized decision rule (PWM) that performs model selection over policy classes. They derive oracle inequalities and regret bounds that show how the procedure adaptively balances approximation and estimation error when the policy class is large or complex. \citet{Sun2026EWMConstraints} studies empirical welfare maximization under budget constraints when costs are unknown and must be estimated. The author highlights a fundamental trade-off between welfare efficiency and feasibility, proving that no statistical rule can achieve both uniformly over a rich class of data-generating processes. It also proposes alternative decision rules, including conservative and trade-off rules, to handle budget uncertainty. Relatedly, \citet{CarneiroLeeWilhelm2020} consider optimal data collection under resource limitations. All these contributions emphasize that budget constraints fundamentally alter the structure of the policy problem, especially when costs are heterogeneous or must be estimated.

\bigskip
\noindent
From a methodological perspective, our contribution differs substantially from the previous literature on empirical welfare maximization under constrains in three key ways: 

\begin{itemize}
\item First, we introduce an additional minimum coverage requirement, which interacts nontrivially with the budget constraint. This leads to a richer Lagrangian structure in which the optimal policy is characterized by an \textit{affine} threshold involving two shadow prices, rather than a single cutoff. Unlike the budget-only case, the coverage constraint induces a heterogeneous shift in the decision boundary, which plays a central role in both the structure and the analysis of the optimal rule.

\item Second, instead of focusing mainly on a population-level characterization of the optimal policy under known primitives, we study the problem of constructing and implementing such policies in practice. We show that the empirical counterpart of constrained optimal rules naturally gives rise to a discrete optimization problem that can be formulated as a constrained \textit{knapsack-type} problem. This perspective makes explicit the computational challenges that are implicit in the original framework.

\item Third, we analyze tractable policy rules based on ranking and provide theoretical guarantees for their performance. In particular, we study their approximation properties relative to the optimal allocation and establish conditions under which simple, implementable rules achieve near-optimality. In this sense, while the existing literature characterizes statistically the optimal policy, our contribution is to bridge the gap between characterization and implementation, providing both a computational framework and performance guarantees for feasible policy learning under more than one constraint.
\end{itemize}

\bigskip
\noindent
Our work is also related to the broader literature on combinatorial optimization and its interaction with machine learning \citep{BengioLodiProuvost2018,ZhangEtAl2022,GarnAmirghasemi2025}. While this literature focuses primarily on algorithmic and computational aspects, we show that similar structures arise naturally in policy learning problems with realistic constraints. In particular, the knapsack formulation emerges directly from the economic primitives of the model.

\bigskip
\noindent
Finally, our analysis connects to statistical learning theory through the use of \textit{margin conditions} (more later). Following the intuition of \citet{Tsybakov2004} and \citet{AudibertTsybakov2007}, we show that misallocation errors induced by RC approximate policies are concentrated near the decision boundary. This allows us to characterize the welfare loss of simple rules, such as our proposed ranking-based policies, and to establish conditions under which they are close to the optimal.

\section{The constrained welfare maximization problem}

\subsection{Potential outcomes framework}

Let $Z = (X,Y,W)$ denote a random vector defined on a probability space, where $X \in \mathcal{X}$ is a vector of covariates, $W \in \{0,1\}$ is the treatment indicator, and $Y \in \mathbb{R}$ is the outcome. We assume i.i.d. sampling:
\[
Z_1,\dots,Z_n \overset{i.i.d.}{\sim} P
\]
where $P$ is the true joint distribution from which data are drawn.
By $P_{X}$, let us indicate the marginal distribution of $X$. We define two potential outcomes $Y(0)$ and $Y(1)$ satisfying the consistency relation (\textit{Rubin's identity}) $Y=Y(0)+W \cdot [Y(1)-Y(0)]$. We make two standard assumptions.

\begin{assumption}[\textit{Unconfoundedness}]
Conditional on the observed covariates $X$, the treatment assignment $W$ is independent of the potential outcomes $Y(0)$ and $Y(1)$:
\[
(Y(0),Y(1)) \perp W \mid X
\]
\end{assumption}

\begin{assumption}[\textit{Overlap or Positivity}]
For every covariate profile $x$ with positive probability, the probability of receiving treatment is bounded away from zero and one. Formally, there exists a constant $\eta \in (0,1/2)$ such that
\[
\eta \le e(x) := P(W=1 \mid X=x) \le 1-\eta
\quad \text{for $P_X$-a.e.\ } x.
\]
\end{assumption}

\noindent
Given these assumptions, we define the (population) conditional average treatment effect (CATE) as
\[
\tau(x)=\mathbb{E}[Y(1)-Y(0)\mid X=x]
\]
which measures the benefit of the policy for units characterized by features \textit{x}. \newline

\noindent
Finally, let $c:\mathcal{X}\to(0,\infty)$ be a measurable cost function. In our baseline analysis $c(x)$ is treated as known. Extensions to estimated costs can be handled similarly by adding a nuisance component (not in this paper).

\subsection{Policies and constrained welfare}

We begin by defining the key building blocks of our framework.

\begin{definition}[Policy]
A policy is a measurable map $\pi:\mathcal{X}\to\{0,1\}$. We denote by $\Pi$ a given class of policies.
\end{definition}

\noindent
Replacing $W$ with $\pi(X)$ in the Rubin's identity, yields
\[
Y(\pi(X)) = Y(0) + \pi(X)\,[Y(1)-Y(0)].
\]
Taking expectation, the \textit{welfare} generated by policy $\pi$ is
\[
\mathbb{E}[Y(\pi(X))]
=
\mathbb{E}\big[Y(0)\big]
+
\mathbb{E}\big[\pi(X)\,(Y(1)-Y(0))\big].
\]

\noindent
Applying the law of iterated expectations with respect to $X$, we obtain
\[
\mathbb{E}\big[\pi(X)\,(Y(1)-Y(0))\big]
=
\mathbb{E}\Big[
\mathbb{E}\big[\pi(X)\,(Y(1)-Y(0)) \mid X\big]
\Big] = \mathbb{E}\Big[
\pi(X)\,\mathbb{E}\big[Y(1)-Y(0)\mid X\big]
\Big]
\]
implying that
\[
\mathbb{E}[Y(\pi(X))]
=
\mathbb{E}[Y(0)]
+
\mathbb{E}\big[\pi(X)\tau(X)\big].
\]

\medskip
\noindent
Since $\mathbb{E}[Y(0)]$ does not depend on the policy $\pi$, maximizing $\mathbb{E}[Y(\pi)]$ is equivalent to maximizing
\[
V(\pi(X))=\mathbb{E}[\pi(X)\tau(X)].
\]

\begin{definition}[Population welfare]
For any $\pi\in\Pi$, define the population welfare as
\[
V(\pi(X))=\mathbb{E}[\pi(X)\tau(X)].
\]
\end{definition}
\noindent
which clearly depends on the policy and on the CATE. \newline

\noindent
We impose two constraints, a \textit{budget} and a \textit{coverage} constraint.

\begin{enumerate}
\item \textit{Budget constraint}. Assume that we have a total budget amount of $B$. If the number of units is $N$, the budget per person is $C = B/N$, with $C>0$. The budget constraint is defined as
\[
\mathbb{E}[\pi(X)c(X)] \le C.
\]

\item \textit{Coverage constraint}. Assume that we want to treat at least a certain percentage $\rho$ of the population, with $\rho\in(0,1)$. The coverage constraint is defined as
\[
\mathbb{E}[\pi(X)] \ge \rho.
\]

\end{enumerate}

\begin{definition}[Feasible policy set and constrained value functional]
Given a policy class $\Pi$, define the \emph{feasible set} under distribution $P$ as
\[
\Pi_{\mathrm{feas}}(P)
=
\left\{
\pi \in \Pi :
\mathbb{E}[\pi(X)c(X)] \le C,
\ \mathbb{E}[\pi(X)] \ge \rho
\right\},
\]
\end{definition}
\noindent
that is, the subset of policies satisfying the budget constraint and the minimum coverage requirement.

\bigskip

\noindent
Finally, define the \emph{constrained value functional} as
\[
\Psi(P)
=
\sup_{\pi \in \Pi_{\mathrm{feas}}(P)} V(\pi),
\]
Any maximizer, when it exists, is denoted
\[
\pi^* \in \arg\max_{\pi \in \Pi_{\mathrm{feas}}(P)} V(\pi)
\]
to be found out within the set $\Pi_{\mathrm{feas}}(P)$, with 
\[
\Psi(P) = V_P\big(\pi^*(P)\big).
\]

\subsection{Feasibility and interior points}

Feasibility requires that the budget constraint be sufficiently \textit{slack} (i.e. sufficiently large) to accommodate the minimum coverage requirement, ensuring that a mass of at least $\rho$ can be treated.

\begin{assumption}[\textit{\textit{Nonemptiness}}]
The feasible policy set is nonempty, that is
\[
\Pi_{\mathrm{feas}}(P)\neq\varnothing.
\]
\end{assumption}
\noindent
This condition requires that the budget and minimum coverage constraints be jointly compatible under the distribution $P$, so that at least one admissible policy satisfies both restrictions. \newline

\noindent
To proceed with value-function maximization, we rely on \textit{strong duality}, ensuring that the \textit{primal} (original) and \textit{dual} (Lagrangian) problems coincide in optimal value. 
We impose a \textit{Slater condition}, requiring the existence of a strictly feasible policy that satisfies the constraints with strict inequalities, so that the feasible set has a nonempty interior.

\begin{assumption}[Slater condition]
There exists a policy $\tilde{\pi} \in \Pi$ such that
\[
\mathbb{E}[\tilde{\pi}(X)c(X)] < C
\quad \text{and} \quad
\mathbb{E}[\tilde{\pi}(X)] > \rho.
\]
\end{assumption}

\noindent
This assumption requires the existence of at least one strictly feasible policy satisfying both constraints with slack. In particular, the budget constraint and the minimum coverage requirement are not simultaneously binding at the boundary of the admissible set.


\section{Population Lagrangian characterization}

\noindent
To better understand the structure of our welfare constrained optimization problem, and thus characterize the optimal policy solution, we derive a population-level primal-dual characterization. Based on the assumptions made in the previous section, we introduce Lagrange multipliers for the budget and coverage constraints, converting the constrained welfare maximization problem into a classical Lagrangian formulation. This representation reveals that, when the policy class is sufficiently rich, the optimal rule can be characterized \textit{pointwise} and takes the form of an \textit{affine threshold} in the adjusted treatment effect. The resulting structure provides a transparent economic interpretation of the constrained optimum and simplifies further the analysis.

\subsection{Lagrangian formulation}

Rewrite the two constraints as follows (for notational convenience, we suppress the dependence on $X$)
\[
g_1(\pi)=\mathbb{E}[\pi c]-C\le 0,
\qquad
g_2(\pi)=\rho-\mathbb{E}[\pi]\le 0.
\]
For multipliers $(\lambda,\nu)\in\mathbb{R}_+^2$, define the Lagrangian
\begin{align}
\mathcal{L}(\pi,\lambda,\nu)
&=
\mathbb{E}[\pi\tau]
-\lambda(\mathbb{E}[\pi c]-C)
-\nu(\rho-\mathbb{E}[\pi]) \nonumber\\
&=
\mathbb{E}\big[\pi(X)\big(\tau(X)-\lambda c(X)+\nu\big)\big]
+\lambda C-\nu\rho.
\label{eq:L}
\end{align}
Define the \textit{dual function} as
\[
g(\lambda,\nu)=\sup_{\pi\in\Pi}\mathcal{L}(\pi,\lambda,\nu),
\]
that is, the maximum value of the Lagrangian over policies for given multipliers $(\lambda,\nu)$. The associated dual problem
\[
\inf_{\lambda,\nu\ge0} g(\lambda,\nu)
\]
searches for the tightest upper bound on the optimal constrained value and, when attained, yields the optimal Lagrange multipliers. \newline

\noindent
If $\Pi$ is sufficiently rich, the inner maximization is pointwise. Indeed, since the Lagrangian is linear in $\pi(x)$ and $\pi(x)\in\{0,1\}$ can be chosen independently across $x$, the maximization reduces to setting $\pi(x)=1$ whenever $\tau(x)-\lambda c(x)+\nu\ge0$, and $\pi(x)=0$ otherwise. \newline

\begin{lemma}[Pointwise maximization]
Suppose $\Pi$ contains all measurable $\{0,1\}$-valued functions. Fix $(\lambda,\nu)\ge0$. Then any maximizer of $\pi\mapsto \mathcal{L}(\pi,\lambda,\nu)$ satisfies
\[
\pi_{\lambda,\nu}(x)=\mathbbm{1}\{\tau(x)-\lambda c(x)+\nu\ge0\}
\quad \text{for $P_X$-a.e.\ } x.
\]
\end{lemma}

\begin{proof}
For fixed $(\lambda,\nu)$, the integrand in \eqref{eq:L} is
$\pi(x)\big(\tau(x)-\lambda c(x)+\nu\big)$.
Since $\pi(x)\in\{0,1\}$ pointwise, the maximizer chooses $\pi(x)=1$ whenever the coefficient is nonnegative and $\pi(x)=0$ otherwise. Measurability follows from measurability of $\tau$ and $c$.
\end{proof}

\noindent
The following theorem provides a formal characterization of the primal-dual structure, building on the preceding steps. \newline

\begin{theorem}[Primal--dual characterization]
Assume nonemptiness and Slater. Suppose $\Pi$ contains all measurable binary rules. Then strong duality holds:
\[
\sup_{\pi\in\Pi_{\mathrm{feas}}(P)}V(\pi)
=
\inf_{\lambda,\nu\ge0}\sup_{\pi\in\Pi}\mathcal{L}(\pi,\lambda,\nu).
\]
Moreover, there exist optimal multipliers $(\lambda^*,\nu^*)$ and an optimal policy $\pi^*$ such that
\[
\pi^*(x)=\mathbbm{1}\{\tau(x)-\lambda^*c(x)+\nu^*\ge0\}
\quad \text{for $P_X$-a.e.\ } x,
\]
and complementary slackness holds:
\[
\lambda^*(\mathbb{E}[\pi^*c]-C)=0,
\qquad
\nu^*(\rho-\mathbb{E}[\pi^*])=0.
\]
\end{theorem}

\begin{proof}
We proceed in four steps.

\medskip
\noindent
\textit{Step 1: Convex relaxation.}
Let $\tilde{\Pi}$ denote the set of measurable functions $\pi:\mathcal{X}\to[0,1]$. 
This set is convex, and can be interpreted as the class of randomized policies assigning treatment with probability $\pi(x)$. Consider the relaxed problem
\[
\sup_{\pi\in\tilde{\Pi}} \mathbb{E}[\pi(X)\tau(X)]
\]
subject to
\[
\mathbb{E}[\pi(X)c(X)] \le C, 
\qquad
\mathbb{E}[\pi(X)] \ge \rho.
\]
This is a convex optimization problem with a linear objective and linear constraints.

\medskip
\noindent
\textit{Step 2: Strong duality.}
Under Assumptions 3 (nonemptiness) and 4 (Slater condition), the feasible set has a strictly interior point. 
Therefore, by standard results in convex analysis (see, e.g., \cite{Rockafellar1970}), strong duality holds for the relaxed problem:
\[
\sup_{\pi\in\tilde{\Pi}_{\mathrm{feas}}(P)} V(\pi)
=
\inf_{\lambda,\nu\ge0} \sup_{\pi\in\tilde{\Pi}} \mathcal{L}(\pi,\lambda,\nu).
\]
Moreover, there exist optimal multipliers $(\lambda^*,\nu^*)$ attaining the dual optimum.

\medskip
\noindent
\textit{Step 3: Pointwise maximization and extremality.}
For fixed $(\lambda,\nu)$, the Lagrangian can be written as
\[
\mathcal{L}(\pi,\lambda,\nu)
=
\mathbb{E}\bigl[\pi(X)(\tau(X)-\lambda c(X)+\nu)\bigr] + \lambda C - \nu\rho.
\]
Since the problem is linear in $\pi(x)$ and $\pi(x)\in[0,1]$ pointwise, the maximizer satisfies
\[
\pi_{\lambda,\nu}(x) \in 
\arg\max_{t\in[0,1]} t(\tau(x)-\lambda c(x)+\nu),
\]
which yields
\[
\pi_{\lambda,\nu}(x)=
\begin{cases}
1 & \text{if } \tau(x)-\lambda c(x)+\nu > 0,\\
0 & \text{if } \tau(x)-\lambda c(x)+\nu < 0,\\
[0,1] & \text{if } \tau(x)-\lambda c(x)+\nu = 0.
\end{cases}
\]

\noindent
Thus, an optimal solution can be chosen to be almost surely binary. 
In particular, there exists an optimal policy $\pi^*$ of the form
\[
\pi^*(x)=\mathbbm{1}\{\tau(x)-\lambda^*c(x)+\nu^*\ge0\}
\quad \text{for } P_X\text{-a.e. } x.
\]

\medskip
\noindent
\textit{Step 4: KKT conditions.}
The complementary slackness conditions follow from standard Karush--Kuhn--Tucker (KKT) conditions for convex problems with inequality constraints:
\[
\lambda^*(\mathbb{E}[\pi^*c]-C)=0,
\qquad
\nu^*(\rho-\mathbb{E}[\pi^*])=0.
\]

\noindent
and this completes the proof.
\end{proof} 

\bigskip

\begin{remark}[Economic interpretation of the affine threshold rule]
The affine threshold representation can be written equivalently as
\[
\tau(x)+\nu^* \;\ge\; \lambda^*\,c(x),
\]
which highlights how the two constraints shape the optimal assignment through their shadow prices. 
\end{remark}

\noindent
The multiplier $\lambda^*$ can be interpreted as the marginal value of relaxing the budget constraint (a ``price'' per unit of cost), while $\nu^*$ is the marginal value of relaxing the minimum coverage requirement (a ``price'' per treated unit). Hence the rule treats an individual with covariates $x$ whenever the benefit $\tau(x)$, adjusted by the coverage shadow price $\nu^*$, exceeds the cost scaled by the budget shadow price $\lambda^*$. \newline

\noindent
Two limiting cases are particularly informative. If the coverage constraint is slack, then $\nu^*=0$ and the decision rule reduces to
\[
\tau(x)\ge \lambda^* c(x)
\quad \Longleftrightarrow \quad
\frac{\tau(x)}{c(x)} \ge \lambda^*,
\]
a standard \textit{cost-effectiveness criterion}: individuals are treated whenever the gain per unit cost exceeds the cutoff $\lambda^*$. In this case, the budget constraint alone pins down $\lambda^*$ so that expected spending is exactly at the budget when the budget is binding.

\bigskip
\noindent
Conversely, when the minimum coverage constraint binds, $\nu^*>0$ shifts the threshold upward uniformly in $x$, effectively relaxing the selectivity of the rule: some units with relatively low (or even slightly negative) $\tau(x)$ may be treated in order to satisfy the coverage requirement. This highlights a fundamental trade-off induced by the constraint: coverage can force treatment of marginal units that would not be selected under a purely budget-constrained welfare-maximizing rule.

\bigskip
\noindent
More generally, the pair $(\lambda^*,\nu^*)$ summarizes the entire impact of the institutional constraints on the optimal policy. The budget constraint determines how strongly costs are penalized (through $\lambda^*$), whereas the coverage constraint determines how much the planner is willing to ``pay'' in welfare terms to guarantee a minimum treated mass (through $\nu^*$). 

\section{Knapsack-type equivalence}

While the Lagrangian characterization seen above delivers a transparent description of the optimal policy, it does not directly provide a closed-form solution for the optimal multipliers. More importantly, in the empirical setting, the problem becomes inherently \textit{discrete}, requiring the selection of a subset of units. This naturally leads to a \textit{combinatorial problem}, which can be formulated as a \textit{knapsack-type problem} and solved computationally. \newline

\noindent
This section shows that our constrained welfare maximization problem can indeed be reinterpreted as a \textit{knapsack-type selection problem}. At the population level, the policymaker’s decision reduces to selecting a measurable subset of covariate values that maximizes total gain subject to budget and minimum coverage constraints. At the sample level, the empirical counterpart becomes a finite-dimensional 0--1 optimization problem, formally equivalent to a knapsack problem with an additional cardinality constraint. This equivalence provides both economic intuition and computational insight, linking constrained policy learning to a well-studied class of combinatorial optimization problems \citep{BengioLodiProuvost2018,ZhangEtAl2022,GarnAmirghasemi2025}.

\subsection{Continuous knapsack representation}

Let $P_X$ be the marginal distribution of $X$. We define the marginal measure
\[
\mu(A) := P_X(A)
\]
for every measurable set $A \subset \mathcal{X}$. Given a policy $\pi$, define its treated region
\[
S_\pi=\{x:\pi(x)=1\}.
\]
Conversely, any measurable $S\subset\mathcal{X}$ defines $\pi_S=\mathbbm{1}\{x\in S\}$.

\begin{proposition}[Set selection form]
The problem
\[
\sup_{\pi\in\Pi_{\mathrm{feas}}(P)} \mathbb{E}[\pi(X)\tau(X)]
\]
is equivalent to
\begin{equation}
\sup_{S\subset\mathcal{X}}
\int_S \tau(x)\,d\mu(x)
\quad \text{s.t.}\quad
\int_S c(x)\,d\mu(x)\le C,
\quad
\mu(S)\ge\rho.
\label{eq:cont_knapsack}
\end{equation}
\end{proposition}

\begin{proof}
Substitute $\pi=\mathbbm{1}_S$. Then $\mathbb{E}[\pi\tau]=\int_S\tau d\mu$, $\mathbb{E}[\pi c]=\int_S c d\mu$, and $\mathbb{E}[\pi]=\mu(S)$. The mapping $S\leftrightarrow \pi$ is bijective up to $\mu$-null sets.
\end{proof}

\begin{proposition}[Continuous knapsack equivalence and affine structure]
Problem \eqref{eq:cont_knapsack} is a continuous knapsack problem with an additional minimum-mass (coverage) constraint. Moreover, under strong duality, any optimal solution admits an affine threshold representation of the form
\[
S^*
=
\{x\in\mathcal{X}:
\tau(x)-\lambda^* c(x)+\nu^* \ge 0\},
\]
for some optimal multipliers $(\lambda^*,\nu^*)\in\mathbb{R}_+^2$.
\end{proposition}

\begin{proof}
Consider the constrained maximization problem (\ref{eq:cont_knapsack}). This problem is structurally equivalent to a knapsack problem in which:
(i) each point $x$ represents an item,
(ii) $\tau(x)$ is its value,
(iii) $c(x)$ is its weight,
(iv) $C$ is the total capacity,
and (v) the additional constraint $\mu(S)\ge\rho$ imposes a minimum mass requirement. Let introduce the Lagrange multipliers $(\lambda,\nu)\in\mathbb{R}_+^2$ associated with the budget and coverage constraints. The Lagrangian is
\begin{align}
\mathcal{L}(S,\lambda,\nu)
&=
\int_S \tau(x)\,d\mu(x)
-\lambda\Big(\int_S c(x)\,d\mu(x)-C\Big)
-\nu\big(\rho-\mu(S)\big).
\end{align}

\noindent
Rearranging terms yields
\begin{align}
\mathcal{L}(S,\lambda,\nu)
&=
\int_S \big(\tau(x)-\lambda c(x)+\nu\big)\,d\mu(x)
+\lambda C - \nu\rho.
\label{eq:L_set_prop}
\end{align}

\noindent
For fixed $(\lambda,\nu)$, maximization over measurable sets $S$ is pointwise. Indeed, each $x\in\mathcal{X}$ contributes
\[
\tau(x)-\lambda c(x)+\nu
\]
to the integrand if and only if $x\in S$. Hence, to maximize $\mathcal{L}(S,\lambda,\nu)$, it is optimal to include $x$ in $S$ whenever this marginal contribution is nonnegative, and to exclude it otherwise. Therefore, any maximizer satisfies
\[
S_{\lambda,\nu}
=
\{x\in\mathcal{X}:
\tau(x)-\lambda c(x)+\nu \ge 0\}.
\]

\noindent
Given standard regularity conditions ensuring strong duality (e.g., nonemptiness and Slater condition), there exist optimal multipliers $(\lambda^*,\nu^*)$ such that the optimal primal solution coincides with the maximizer of the Lagrangian evaluated at these multipliers. Consequently, the optimal set is
\[
S^*
=
\{x\in\mathcal{X}:
\tau(x)-\lambda^* c(x)+\nu^* \ge 0\}.
\]

\noindent
This establishes that the solution to the constrained policy problem admits an \textit{affine threshold structure} in $(\tau(x),c(x))$.
\end{proof}

\subsection{A 0-1 knapsack reduction with cardinality}

We now move from the population problem to its empirical/sample counterpart. Although the population formulation provides a structural characterization of the optimal policy, implementation requires working with finite data. In this setting, policy learning reduces to selecting a subset of observed units so as to maximize estimated gains subject to budget and coverage constraints. Crucially, this problem admits an exact finite-sample reformulation as a finite-dimensional combinatorial optimization problem. The following result shows that it is equivalent to a \textit{0--1 knapsack problem} with an additional \textit{cardinality constraint}.

\bigskip
\noindent
Let $\hat\Gamma_i$ be any estimation of $\tau(X_i)$ (e.g., doubly robust). Define the sample welfare and constraints as
\[
\hat V_n(\pi)=\frac{1}{n}\sum_{i=1}^n \pi(X_i)\hat\Gamma_i,
\quad
\hat B_n(\pi)=\frac{1}{n}\sum_{i=1}^n \pi(X_i)c(X_i),
\quad
\hat M_n(\pi)=\frac{1}{n}\sum_{i=1}^n \pi(X_i).
\]
The empirical constrained maximization is

\begin{align}
\max_{\pi\in\Pi}\ \hat V_n(\pi)
\quad \text{s.t.}\quad
\hat B_n(\pi)\le C,\ \hat M_n(\pi)\ge \rho.
\label{eq:emp_prop1}
\end{align}

\begin{theorem}[Finite-sample knapsack equivalence]
If $\Pi$ contains all labelings on $\{X_1,\dots,X_n\}$ (that is, for any vector of treatment decisions across the $n$ observed units, there exists a policy in $\Pi$ that generates exactly that assignment on the sample), then the empirical constrained problem (\ref{eq:emp_prop1}) is exactly equivalent to the finite-dimensional program
\begin{align}
\max_{d\in\{0,1\}^n}\ \sum_{i=1}^n \hat\Gamma_i d_i
\quad \text{s.t.}\quad
\sum_{i=1}^n c(X_i)d_i \le nC,
\quad
\sum_{i=1}^n d_i \ge n\rho,
\label{eq:emp_prop2}
\end{align}
where $d_i$ denotes the treatment decision for unit $i$. This is a 0--1 knapsack problem with an additional cardinality (minimum coverage) constraint.
\end{theorem}

\begin{proof}
Let $d_i=\pi(X_i)$. Any feasible $\pi$ induces a feasible $d$, and any $d$ defines a policy on the sample points. Substitution yields the stated integer program.
\end{proof}

\section{Linear programming relaxation}

The empirical policy learning problem (\ref{eq:emp_prop2}) takes the form of a combinatorial 0--1 knapsack program, which is in general computationally demanding to solve due to its discrete nature. A natural approach is to consider its \textit{linear programming (LP) relaxation}, where binary decisions are replaced by continuous ones. This relaxation is significantly easier to analyze and solve, and admits a convenient Lagrangian characterization in terms of threshold rules.

\bigskip
\noindent
However, replacing a discrete problem with a continuous one raises a fundamental question: how much is lost by relaxing the integrality constraint? The purpose of this section is to show that, under mild conditions, the loss is negligible in large samples. In particular, we establish that the difference between the optimal 0--1 solution and its LP relaxation vanishes in per-capita terms. 

\bigskip
\noindent
This result provides a formal justification for analyzing the constrained policy learning problem through its LP relaxation, and underpins the validity of simple threshold-based (greedy) rules as asymptotically optimal solutions. As a result, any threshold rule that is asymptotically optimal for the LP relaxation is also asymptotically optimal for the original 0--1 problem. This establishes a precise bridge between Lagrangian/KKT characterizations and implementable greedy allocation rules in large samples.

\bigskip
\noindent
In order to show this finding, let us re-formulate our constrained problem in its population form, following knapsack-style notation, thus obtaining: 
\begin{equation}
\label{eq:01_knap_card}
\mathrm{OPT}^{01}_n
:=
\max_{\pi\in\{0,1\}^n}
\sum_{i=1}^n v_i\pi_i
\quad
\text{s.t.}
\quad
\sum_{i=1}^n w_i\pi_i \le W_n,
\qquad
\sum_{i=1}^n \pi_i \ge K_n,
\end{equation}
where 
$v_i := \hat\tau(X_i)\in\mathbb R$, $w_i := c(X_i)\in(0,\infty)$, $W_n := nC$, and $K_n := \lceil n\rho\rceil$ with the minimum coverage requirement $\rho\in(0,1)$, and at least $K_n$ units to be assigned to treatment.

\medskip
\noindent
This is a 0--1 knapsack problem with an additional minimum-cardinality constraint: each unit $i$ is an ``item'', $v_i$ is its value, $w_i$ is its weight, the budget $W_n$ is the knapsack capacity, and the coverage condition imposes that at least $K_n$ items be selected.

\medskip
\noindent
As said, problem \eqref{eq:01_knap_card} is combinatorial. It is thus natural to consider its linear programming relaxation, obtained by replacing the binary restrictions $\pi_i\in\{0,1\}$ with the continuous constraints $z_i\in[0,1]$. The relaxed variable $z_i$ can be interpreted as a \textit{fractional} treatment assignment. The resulting LP problem is
\begin{equation}
\label{eq:LP_relax}
\mathrm{OPT}^{\mathrm{LP}}_n
:=
\max_{z\in[0,1]^n}
\sum_{i=1}^n v_i z_i
\quad
\text{s.t.}
\quad
\sum_{i=1}^n w_i z_i \le W_n,
\qquad
\sum_{i=1}^n z_i \ge K_n.
\end{equation}
This relaxation preserves the linear objective and the linear constraints, while enlarging the feasible set from binary allocations to fractional ones.

\medskip
\noindent
The first basic comparison between the two formulations is immediate: the LP relaxation can only improve upon, or at least match, the value of the original 0--1 problem. Indeed, every binary feasible allocation is also feasible for the relaxed problem. Hence,
\begin{equation}
\label{eq:chain}
\mathrm{OPT}^{01}_n \le \mathrm{OPT}^{\mathrm{LP}}_n.
\end{equation}

\begin{lemma}[LP relaxation dominates the 0--1 problem]
\label{lem:lp_dominates}
Let
\[
\mathrm{OPT}^{01}_n
:=
\max_{\pi\in\{0,1\}^n}
\sum_{i=1}^n v_i\pi_i
\quad
\text{s.t.}
\quad
\sum_{i=1}^n w_i\pi_i \le W_n,
\qquad
\sum_{i=1}^n \pi_i \ge K_n,
\]
and let
\[
\mathrm{OPT}^{\mathrm{LP}}_n
:=
\max_{z\in[0,1]^n}
\sum_{i=1}^n v_i z_i
\quad
\text{s.t.}
\quad
\sum_{i=1}^n w_i z_i \le W_n,
\qquad
\sum_{i=1}^n z_i \ge K_n.
\]
Then
\[
\mathrm{OPT}^{01}_n \le \mathrm{OPT}^{\mathrm{LP}}_n.
\]
\end{lemma}

\begin{proof}
Let
\[
\mathcal F_{01}
:=
\Big\{
\pi\in\{0,1\}^n:
\sum_{i=1}^n w_i\pi_i\le W_n,\ 
\sum_{i=1}^n \pi_i\ge K_n
\Big\}
\]
denote the feasible set of the 0--1 problem, and let
\[
\mathcal F_{\mathrm{LP}}
:=
\Big\{
z\in[0,1]^n:
\sum_{i=1}^n w_i z_i\le W_n,\ 
\sum_{i=1}^n z_i\ge K_n
\Big\}
\]
denote the feasible set of the LP relaxation. Since $\{0,1\}\subset[0,1]$, every vector $\pi\in\mathcal F_{01}$ also belongs to $\mathcal F_{\mathrm{LP}}$. Therefore,
\[
\mathcal F_{01}\subseteq \mathcal F_{\mathrm{LP}}.
\]
Both optimization problems maximize the same linear objective function,
\[
\phi(u):=\sum_{i=1}^n v_i u_i.
\]
Taking the supremum of the same objective over a larger feasible set cannot decrease its optimal value. Hence,
\[
\sup_{\pi\in\mathcal F_{01}} \sum_{i=1}^n v_i\pi_i
\;\le\;
\sup_{x\in\mathcal F_{\mathrm{LP}}} \sum_{i=1}^n v_i x_i,
\]
which is exactly
\[
\mathrm{OPT}^{01}_n \le \mathrm{OPT}^{\mathrm{LP}}_n.
\]
\end{proof}

\subsection{Lagrangian scores and greedy threshold rule}

For multipliers $(\lambda,\nu) \in \mathbb{R}_+^2$, define the sample Lagrangian of problem (\ref{eq:LP_relax}) as
\[
\mathcal{L}_n(x,\lambda,\nu)
=
\sum_{i=1}^n x_i \big( v_i - \lambda w_i + \nu \big)
+ \lambda W_n
- \nu K_n.
\]

\noindent
For fixed $(\lambda,\nu)$, maximizing $\mathcal{L}_n(\cdot,\lambda,\nu)$ over $x \in [0,1]^n$
is pointwise:
\[
x_i(\lambda,\nu)
\in
\arg\max_{t \in [0,1]}
t \big( v_i - \lambda w_i + \nu \big)
=
\begin{cases}
1 & \text{if } v_i - \lambda w_i + \nu > 0, \\
0 & \text{if } v_i - \lambda w_i + \nu < 0, \\
[0,1] & \text{if } v_i - \lambda w_i + \nu = 0.
\end{cases}
\]

\noindent
This motivates the greedy threshold (score-based) 0--1 rule:
\begin{equation}
\label{eq:greedy_threshold}
\pi_i^{G}(\lambda,\nu)
:=
\mathbf{1}\{ v_i - \lambda w_i + \nu \ge 0 \}.
\end{equation}

\noindent
In practice, one chooses $(\lambda_n,\nu_n)$ so that the induced threshold rule is approximately feasible:
\begin{equation}
\label{eq:approx_feas}
\sum_{i=1}^n w_i \pi_i^{G}(\lambda_n,\nu_n) \le W_n,
\qquad
\sum_{i=1}^n \pi_i^{G}(\lambda_n,\nu_n) \ge K_n,
\end{equation}
and complementary slackness is nearly satisfied.

\subsection{Extreme point characterization of the LP relaxation}

A key step in establishing the asymptotic equivalence between the 0--1 problem and its LP relaxation is to understand the structure of optimal solutions to the relaxed problem. In particular, the potential advantage of the LP formulation over the combinatorial one arises from the possibility of assigning fractional values to some components. 

\bigskip
\noindent
The following lemma shows that this advantage is in fact highly limited: any optimal extreme point of the LP relaxation can have \textit{at most two fractional components}. As a consequence, the difference between the LP solution and a feasible binary solution is confined to a finite number of units, independently of the sample size. This structural property plays a crucial role in proving that the integrality gap is uniformly bounded and vanishes in per-capita terms.

\bigskip

\begin{lemma}[At most two fractional items in the LP optimum]
\label{lem:two_fractional}
Any extreme point solution $z^{\mathrm{LP}}_n$ of \eqref{eq:LP_relax} has at most two indices
$i$ such that $z^{\mathrm{LP}}_{n,i} \in (0,1)$.
\end{lemma}

\begin{proof}
The feasible set of \eqref{eq:LP_relax} is a polytope in $\mathbb{R}^n$ defined by:

(i) box constraints $0 \le z_i \le 1$ for $i=1,\dots,n$,

(ii) one budget inequality $\sum_i w_i z_i \le W_n$,

(iii) one coverage inequality $-\sum_i z_i \le -K_n$.

\bigskip
\noindent
The feasible set of \eqref{eq:LP_relax} is a polytope in $\mathbb{R}^n$, obtained by intersecting the hypercube $[0,1]^n$ with two global linear constraints corresponding to the budget and coverage requirements. At an extreme point, at least $n$ linearly independent constraints are tight.
Since there are only two global constraints (budget and coverage), at least $n-2$ tight constraints must come from the box constraints.
Hence at least $n-2$ coordinates are at a bound (0 or 1), and at most two
coordinates can lie strictly in $(0,1)$.
\end{proof}

\subsection{Constant integrality gap}
The importance of Lemma \ref{lem:two_fractional} relies on the fact that the LP optimum differs from a binary solution in at most two coordinates. Hence the advantage of the fractional relaxation is confined to a finite number of marginal items, implying that the integrality gap is uniformly bounded and vanishes in per-capita terms. We prove this statement in what follows. 

\begin{lemma}[Integrality gap is $O(1)$]
\label{lem:gap_O1}
Assume boundedness:
\begin{equation}
\label{eq:bounded}
|v_i| \le V_{\max},
\qquad
0 < w_{\min} \le w_i \le w_{\max}
\quad \text{a.s.}
\end{equation}
Under \eqref{eq:bounded}, for every $n$,
\[
0 \le
\mathrm{OPT}^{\mathrm{LP}}_n
-
\mathrm{OPT}^{01}_n
\le
2 V_{\max}.
\]
\end{lemma}

\begin{proof}
Let $z^{\mathrm{LP}}_n$ be an optimal extreme point of \eqref{eq:LP_relax}.
By Lemma \ref{lem:two_fractional}, at most two coordinates are fractional. Construct a 0--1 vector $\tilde\pi \in \{0,1\}^n$ by rounding those fractional coordinates to either 0 or 1 so as to preserve feasibility. The objective loss relative to $z^{\mathrm{LP}}_n$
is at most the total value contributed by the at most two fractional components, bounded by $2 V_{\max}$ in absolute value. Since $\mathrm{OPT}^{01}_n$ is the best feasible 0--1 objective,
\[
\mathrm{OPT}^{\mathrm{LP}}_n - \mathrm{OPT}^{01}_n
\le
\mathrm{OPT}^{\mathrm{LP}}_n
-
\sum_{i=1}^n v_i \tilde\pi_i
\le
2 V_{\max}.
\]
Nonnegativity follows from \eqref{eq:chain}.
\end{proof}

\subsection{Asymptotic equivalence}

Define the per-capita optima:
\[
\bar V^{01}_n
:=
\frac{1}{n} \mathrm{OPT}^{01}_n,
\qquad
\bar V^{\mathrm{LP}}_n
:=
\frac{1}{n} \mathrm{OPT}^{\mathrm{LP}}_n.
\]

\begin{theorem}[Asymptotic equivalence: 0--1 vs.\ fractional]
\label{thm:asymp_equiv}
Under \eqref{eq:bounded},
\[
0
\le
\bar V^{\mathrm{LP}}_n
-
\bar V^{01}_n
\le
\frac{2 V_{\max}}{n}
\longrightarrow 0.
\]
Hence the 0--1 knapsack with minimum coverage is asymptotically
equivalent (in per-capita value) to its LP relaxation.
\end{theorem}

\begin{proof}
Divide the bound in Lemma \ref{lem:gap_O1} by $n$ and let $n \to \infty$.
\end{proof}

\begin{remark}[Vanishing per-capita integrality gap]
\label{rem:squeeze}
The result in Theorem \ref{thm:asymp_equiv} follows from a simple comparison argument.
Let
\[
a_n := \bar V^{\mathrm{LP}}_n - \bar V^{01}_n,
\qquad
b_n := \frac{2V_{\max}}{n}.
\]
By Lemma \ref{lem:gap_O1}, we have
\[
0 \le a_n \le b_n,
\]
and clearly $b_n \to 0$ as $n \to \infty$.
Hence, by the squeeze (comparison) theorem,
\[
\bar V^{\mathrm{LP}}_n - \bar V^{01}_n \to 0.
\]

\noindent
Intuitively, although the LP relaxation may improve upon the 0--1 solution,
the total improvement is uniformly bounded by a constant ($2V_{\max}$).
When normalized by the sample size $n$, this gain vanishes,
so that the two problems become asymptotically equivalent in per-capita value (that is, by dividing by $n$).
\end{remark}

\subsection{Asymptotic optimality of the greedy threshold rule}

Assume the nondegeneracy condition
\begin{equation}
\label{eq:nondeg}
\mathbb{P}\big( v - \lambda w + \nu = 0 \big) = 0
\quad
\text{for the relevant optimal } (\lambda,\nu).
\end{equation}

\begin{theorem}[Asymptotic optimality of the greedy threshold rule]
\label{thm:greedy_asymp}
Assume \eqref{eq:bounded} and \eqref{eq:nondeg}.
Let $(\lambda_n,\nu_n)$ be multipliers such that the greedy policy
$\pi^G(\lambda_n,\nu_n)$ in \eqref{eq:greedy_threshold}
is feasible and satisfies
\[
\sum_{i=1}^n v_i \pi_i^{G}(\lambda_n,\nu_n)
\ge
\mathrm{OPT}^{\mathrm{LP}}_n - o(n).
\]

\noindent
Then
\[
\frac{1}{n}
\sum_{i=1}^n v_i \pi_i^{G}(\lambda_n,\nu_n)
-
\bar V^{01}_n
\longrightarrow 0.
\]
\end{theorem}

\begin{proof}
By assumption,
\[
\frac{1}{n}
\sum_{i=1}^n v_i \pi_i^{G}(\lambda_n,\nu_n)
\ge
\bar V^{\mathrm{LP}}_n - o(1).
\]

\noindent
By Theorem \ref{thm:asymp_equiv},
\[
\bar V^{\mathrm{LP}}_n
=
\bar V^{01}_n + o(1).
\]

\noindent
Hence
\[
\frac{1}{n}
\sum_{i=1}^n v_i \pi_i^{G}(\lambda_n,\nu_n)
\ge
\bar V^{01}_n - o(1).
\]

\noindent
Feasibility implies the greedy value cannot exceed $\bar V^{01}_n$
by definition of $\mathrm{OPT}^{01}_n$.
Therefore the difference converges to zero.
\end{proof}

\subsection{A practical greedy algorithm}

The previous result establishes that any feasible greedy policy achieving near-optimal performance for the LP relaxation is asymptotically optimal for the original 0--1 problem. The purpose of this section is to provide a constructive procedure to obtain such a policy. In particular, the following algorithm implements a primal--dual search for the Lagrange multiplier $\lambda$ and generates a greedy threshold rule that approximates the LP optimum and satisfies the feasibility conditions required in Theorem \ref{thm:greedy_asymp}.

\bigskip
\noindent
Consider again the empirical problem
\[
\max_{\pi\in\{0,1\}^n}\ \sum_{i=1}^n v_i\pi_i
\quad\text{s.t.}\quad
\sum_{i=1}^n w_i\pi_i \le W_n,\qquad \sum_{i=1}^n \pi_i \ge K_n,
\]
with values $v_i\in\mathbb{R}$ and strictly positive costs $w_i>0$.

\bigskip
\noindent
\textit{Feasibility check}. A necessary and sufficient condition for feasibility is that the minimum cost of selecting $K_n$
items does not exceed the budget:
\[
\sum_{j=1}^{K_n} w_{(j)} \le W_n,
\]
where $w_{(1)}\le \cdots \le w_{(n)}$ are the ordered costs. If it fails, the problem is infeasible.

\bigskip
\noindent
\textit{Greedy-Lagrangian principle}. For $\lambda\ge 0$, define the Lagrangian score
\[
a_i(\lambda):=v_i-\lambda w_i.
\]
For each $\lambda$, we construct a feasible 0--1 allocation $\pi^G(\lambda)$ by:
(i) enforcing coverage by taking the $K_n$ largest scores $a_i(\lambda)$,
(ii) then adding extra units with positive score, in decreasing score order, as long as budget allows. This yields a one-parameter family of feasible policies. We then choose $\lambda$ by bisection
to make the budget approximately tight (frontier solution). Below we show a more detailed procedural presentation of this algorithm. 

\newpage

\begin{algorithm}[h]
\caption{Greedy-Lagrangian with minimum coverage and budget (GLC)}
\label{alg:GLC}
\begin{enumerate}
\item \emph{Input:} arrays $(v_i)_{i=1}^n$, $(w_i)_{i=1}^n$ with $w_i>0$, budget $W_n>0$, coverage $K_n\in\{1,\dots,n\}$,
tolerance $\varepsilon>0$, max iterations $T$, deterministic tie-break rule.

\item \emph{Feasibility check:}
sort indices by increasing $w_i$ and compute $W_{\min K}:=\sum_{j=1}^{K_n} w_{(j)}$.
If $W_{\min K}>W_n$, stop (infeasible).

\item \emph{Choose bisection bracket for $\lambda$:}
set $\lambda_L:=0$.
Choose $\lambda_U$ large enough so that the GLC selection at $\lambda_U$ has cost at most $W_n$.
A safe constructive choice is to increase $\lambda_U$ geometrically until feasibility holds:
initialize $\lambda_U:=1$, and while $\mathrm{Cost}(\pi^G(\lambda_U))>W_n$ set $\lambda_U:=2\lambda_U$.

\item \emph{For $t=1,\dots,T$:}
\begin{enumerate}
\item Compute scores $a_i:=v_i-\lambda_M w_i$, where $\lambda_M:=(\lambda_L+\lambda_U)/2$.

\item Sort indices by decreasing $a_i$ (break ties deterministically, e.g. smaller $w_i$ first, then index).

\item \emph{Coverage step (forced core):}
let $S:=\{i_{(1)},\dots,i_{(K_n)}\}$ be the first $K_n$ indices in the sorted list.
Set $\pi_i:=\mathbf{1}\{i\in S\}$ for all $i$.
Compute $B:=\sum_{i\in S} w_i$ and $V:=\sum_{i\in S} v_i$.

\item \emph{Budget repair if needed:}
if $B>W_n$, then $\lambda_M$ is too small (heavy items are not penalized enough).
Set $\lambda_L:=\lambda_M$ and continue to next iteration.

\item \emph{Greedy augmentation (positive-score add-ons):}
scan the remaining indices in sorted order $i_{(K_n+1)},\dots,i_{(n)}$.
For each such index $j$:
if $a_j\le 0$, stop the scan (all remaining scores are nonpositive);
else if $B+w_j\le W_n$, add it: set $\pi_j:=1$, update $S:=S\cup\{j\}$,
$B:=B+w_j$, $V:=V+v_j$.

\item \emph{Bisection update (budget frontier):}
if $B\le W_n$ and $W_n-B\le \varepsilon W_n$, stop and output $\pi$.
If $B<W_n-\varepsilon W_n$ (too much slack), decrease $\lambda$ to admit more items:
set $\lambda_U:=\lambda_M$.
Else (near-tight already handled), keep the bracket update as above.
\end{enumerate}

\item \emph{Output:} $\pi$ (and optionally $(\lambda_M,B,V)$).
\end{enumerate}
\end{algorithm}

\bigskip
\noindent
Algorithm \ref{alg:GLC} solves the empirical constrained selection problem by combining a Lagrangian relaxation of the budget constraint with an explicit enforcement of the minimum coverage requirement. The core idea is to introduce a one-dimensional dual parameter $\lambda$
interpreted as a shadow price of the budget, and to construct for each $\lambda$
a score
\[
a_i(\lambda) = v_i - \lambda w_i,
\]
which represents the net marginal contribution of unit $i$ once cost is penalized. For fixed $\lambda$, units are ranked by decreasing score.

\bigskip
\noindent
The algorithm enforces feasibility in two stages.
First, the coverage constraint is satisfied by selecting the $K_n$ units
with highest score (the ``forced core'').
Second, additional units with strictly positive score are greedily added
in descending order as long as the budget constraint is not violated.
This guarantees that the resulting allocation is always coverage-feasible
and budget-feasible whenever the problem itself is feasible.

\bigskip
\noindent
The remaining task is to determine the appropriate value of $\lambda$.
This is achieved through a bisection procedure.
Because the total budget usage $B(\lambda)$ induced by the greedy rule
is monotone non-increasing in $\lambda$, the dual parameter can be
adjusted until the resulting allocation lies on (or arbitrarily close to)
the budget frontier.
Intuitively, increasing $\lambda$ penalizes expensive units more heavily,
thereby reducing total expenditure; decreasing $\lambda$ relaxes the
penalty and increases expenditure.

\bigskip
\noindent
Thus, the algorithm implements a primal–dual search along a one–dimensional dual path. It produces a feasible 0--1 allocation that approximates the solution of the linear programming relaxation, and by the integrality gap result, is asymptotically optimal for the original combinatorial problem once normalized by sample size.

\bigskip
\noindent
It is worth comparing LP and GLC computational complexity. The LP relaxation problem is a linear program with $n$ variables and $m=2$ linear constraints (in addition to box constraints). Using standard interior-point methods, the worst-case computational complexity is polynomial, typically of order $O(n^3)$.
In practice, simplex-based solvers often perform well, but their worst-case complexity is exponential. Moreover, generic LP solvers do not fully exploit the specific knapsack structure of the problem.
In the GLC algorithm, each iteration requires sorting the $n$ scores, yielding complexity $O(n \log n)$ per iteration. As the value of $\lambda$ is determined via bisection, and the induced budget usage is monotone in $\lambda$, the number of iterations required to achieve tolerance $\varepsilon$ is
$O(\log(1/\varepsilon))$. Therefore, GLC overall complexity is
$O\big(n \log n \cdot \log(1/\varepsilon)\big),$ making this procedure scalable to large empirical OPL applications. 

\bigskip
\noindent
Moreover, the key difference between LP and GLC lies not only in computational complexity, but also in how each method exploits the structure of the constrained problem. The LP approach solves a global optimization problem using general-purpose methods, whereas the GLC algorithm exploits more in depth the knapsack structure of the problem, reducing the computation to sorting and one-dimensional dual search. Appendix A provides a numerical illustration of the GLC algorithm.

\section{Characterizing a \textit{rank-and-cut} approximation}

The Greedy--Lagrangian procedure described in Algorithm \ref{alg:GLC}
provides a principled way to approximately solve the empirical constrained knapsack problem. It is directly motivated by the dual characterization of the LP relaxation and implements a one-dimensional search over the shadow price $\lambda$ of the budget constraint.

\bigskip
\noindent
This approach is theoretically well-grounded and leads to
asymptotically optimal allocations; however, it requires iterative bisection, repeated sorting of scores, and careful numerical tuning. As a result, its computational cost is $O(n \log n \cdot \log(1/\varepsilon))$.

\bigskip
\noindent
This naturally raises the question of whether a simpler, non-iterative procedure can achieve comparable performance. A particularly appealing alternative is the \textit{rank-and-cut} (RC) rule, which orders units by the \textit{cost-effectiveness ratio} $v_i / w_i$ and selects the largest feasible prefix satisfying the budget and coverage constraints. This procedure requires only a single sorting step and has computational complexity $O(n \log n)$.
The key question is to understand to what extent this substantial computational simplification comes at an acceptable misallocation cost.

\bigskip
\noindent
By characterizing the rank-and-cut solution, this section provides a detailed theoretical analysis of when the rank-and-cut rule can be \textit{approximately} optimal. Our main result shows that the rank-and-cut achieves a comparable approximate per-capita welfare as the optimal policy whenever the heterogeneity induced by the minimum coverage constraint does not significantly distort the ranking. The key mechanism is that, in this case, misallocation is confined to a shrinking neighborhood of the decision boundary, where the marginal welfare contribution vanishes.

\subsection{Setup}

To start with, consider again the Lagrangian of our baseline constrained optimization
\[
\mathcal{L}(\pi,\lambda,\nu)
=
\mathbb{E}\left[(\tau(X)-\lambda c(X)+\nu)\pi(X)\right]
+\lambda B - \nu \rho,
\]
where the optimal policy is
\[
\pi^*(x) = 1\{m^*(x) \ge 0\}
\]
with the \textit{margin} defined as
\[
m^*(x) = \tau(x) - \lambda^* c(x) + \nu^*
\]
We can provide a ratio representation of the optimal solution by dividing by $c(x)$, thus obtaining
\[
\frac{m^*(x)}{c(x)} = \frac{\tau(x)}{c(x)} - \lambda^* + \frac{\nu^*}{c(x)}
\]

\bigskip
\noindent
By defining:
\[
r(x) := \frac{\tau(x)}{c(x)}, \quad
b_{LP}(x) := \lambda^* - \frac{\nu^*}{c(x)}
\]
we have that
\[
\pi^*(x) = 1\{r(x) \ge b_{LP}(x)\}
\]
This is the optimal solution expressed as a threshold rule over $r(x)$. 

\bigskip
\noindent
Alternatively, the rank-and-cut decision rule is
\[
\pi^{RC}_n(x) = 1\{r(x) \ge t^*\},
\]
where it emerges clearly that the key difference between the LP and RC rules stands on the fact that the LP threshold depends on $x$ (that is, it is heterogeneous and thus non-constant), whereas the RC threshold is constant. 

\bigskip
\noindent
The welfare loss induced by the RC rule, can be expressed as
\[
\Delta_n = \frac{1}{n}\sum \tau_i(\pi_i^* - \pi_i^{RC})
\]
This welfare loss arises from a set of misallocations induced by the rank-and-cut (RC) rule relative to the optimal (LP) rule. We want to characterize such welfare loss, and study its behavior.
First, we define the \textit{misallocation region} as
\[
\mathcal{M}_n = \{x : \pi^*(x) \neq \pi^{RC}_n(x)\}
\]
Equivalently:
\[
\mathcal{M}_n =
\left\{
x:
(r(x)-t^*)(r(x)-b_{LP}(x)) < 0
\right\}
\]
This implies:
\[
r(x) \in
[\min\{t^*,b_{LP}(x)\}, \max\{t^*,b_{LP}(x)\})
\]
By recall the definition of margin, we can see that it can be also represented as:
\[
m^*(x) = c(x)(r(x)-b_{LP}(x))
\]
Hence, it follows that:
\[
|m^*(x)| = c(x)\,|r(x)-b_{LP}(x)|
\]
so that, if a misallocation occurs, it means that:
\[
|r(x)-b_{LP}(x)| \le |t^* - b_{LP}(x)|
\]
implying that:
\[
|m^*(x)| \le c(x)\,|t^* - b_{LP}(x)|
\]
Now, assume that the cost function is bounded this way
\[
0 < \underline{c} \le c(x) \le \bar{c}
\]
If we assume \textit{uniform convergence}, that is
\[
\sup_x |t^* - b_{LP}(x)| \le \delta_n
\]
then, it follows that
\[
\mathcal{M}_n \subseteq \{x : |m^*(x)| \le \bar{c}\delta_n\}
\]
This is an important result, suggesting that under thresholds uniform convergence, misallocation lies in a \textit{margin band}, the one around the optimal decision boundary $m^*(x) = 0$.
In other words, any disagreement between the rank-and-cut rule and the optimal policy can only occur for units whose dual margin is small. This is a crucial reduction. Instead of analyzing a complex combinatorial difference between policies, we reduce the problem to understanding the behavior of the margin near zero.

\bigskip
\noindent
An important consequence of this finding is that controlling $\delta_n$ is sufficient to control misallocation. The crucial assumption is:
\[
\sup_x |t^* - b_{LP}(x)| \le \delta_n \to 0,
\]
requiring that a constant threshold approximates a heterogeneous one. This is not automatic and deserves careful interpretation. Indeed, recall:
\[
b_{LP}(x) = \lambda^* - \frac{\nu^*}{c(x)}.
\]
Thus, heterogeneity arises entirely from the term:
\[
\frac{\nu^*}{c(x)}.
\]
One requires that this term does not vary too much across $x$. There are two special cases where this happens: \newline

\noindent
\textit{Case 1: Low cost heterogeneity}. If:
\[
c(x) \approx c_0,
\]
then:
\[
\frac{\nu^*}{c(x)} \approx \frac{\nu^*}{c_0},
\]
so $b_{LP}(x)$ is approximately constant. This is the easiest and more transparent case. \newline

\noindent
\textit{Case 2: Weak coverage constraint}. If:

\[
\nu^* \approx 0,
\]
then:
\[
b_{LP}(x) \approx \lambda^*,
\]
which is constant. In this case, the problem reduces to the standard budget-constrained setting. \newline

\noindent
More generally, the condition can be interpreted as requiring that the coverage constraint does not significantly distort the cost-effectiveness ranking. What happens if this assumption fails? If
$|t^* - b_{LP}(x)|$ does not become small uniformly, then the misallocation region is not confined to a thin band. This implies that misallocation may occur far from the boundary, high-value units may be misclassified, and the welfare gap may not vanish. Thus, the assumption is not merely technical, but it identifies the exact condition under which rank-and-cut remains reliable.

\bigskip
\noindent
Figure \ref{fig:LP_RC_misallocation} shows the misallocation region obtained by comparing the LP and RC (rank-and-cut) decision boundaries. We consider the $(c(x),r(x))$ Cartesian plane, where $r(x)=\tau(x)/c(x)$ denotes the cost-effectiveness ratio. The LP rule is governed by the boundary $b_{LP}(c)=\lambda^*-\frac{\nu^*}{c}$, which is increasing in \(c\) whenever \(\nu^*>0\), since $\frac{d}{dc}b_{LP}(c)=\frac{\nu^*}{c^2}>0$. Hence, under LP, the minimum required cost-effectiveness ratio becomes more demanding as the individual cost increases. 

\bigskip
\noindent
By contrast, the RC rule is based on the constant threshold
$b_{RC}(c)=t^*$, which ignores cost heterogeneity except through the ratio $r(x)$ itself. The two boundaries intersect at the unique crossing point $ c^\times=\frac{\nu^*}{\lambda^*-t^*}$,
obtained by solving $\lambda^*-\frac{\nu^*}{c}=t^*$. This crossing point divides the plane into two qualitatively different misallocation regions.

\bigskip
\noindent
For low-cost units, that is, for $c(x)<c^\times$, one has $
b_{LP}(c(x))<t^*$, so the LP rule is more permissive than RC. Equivalently, there exists a region in which RC treats fewer units than LP. For high-cost units, that is, for $c(x)>c^\times$, the inequality reverses, that is, $b_{LP}(c(x))>t^*$,
so LP becomes more selective than RC, and RC may treat units that LP would reject.

\bigskip
\noindent
This figure therefore makes transparent the structural source of misallocation induced by RC relative to LP. As RC uses a flat cutoff, while LP uses a cost-dependent cutoff, RC cannot adapt to heterogeneity in treatment cost. As a consequence, it tends to be too restrictive in one portion of the cost space and too permissive in another. This visual argument helps clarify that the units' selection discrepancy between RC and LP is characterized by the interaction between treatment-effect heterogeneity and cost heterogeneity.

\begin{figure}[!h]
\centering
\includegraphics[width=0.7\textwidth]{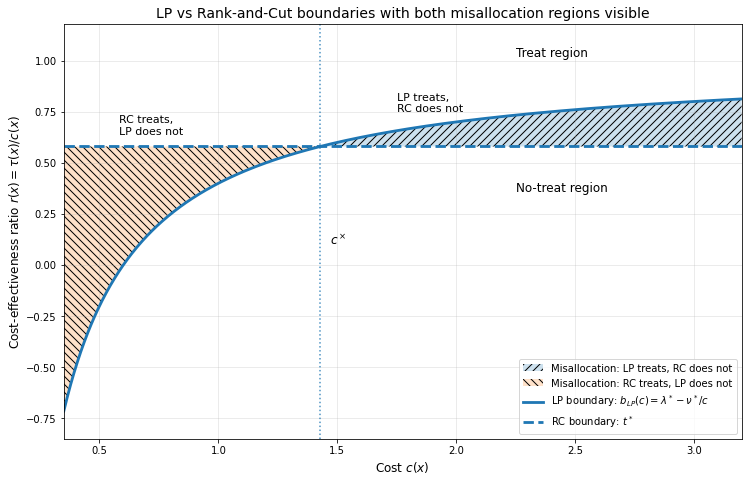}
\caption{Misallocation region: comparison between the LP and Rank-and-Cut decision boundaries: $(c(x),r(x))$ is the plane, where $r(x)=\tau(x)/c(x)$ denotes the cost-effectiveness ratio.}
\label{fig:LP_RC_misallocation}
\end{figure}

\subsection{RC welfare loss bound}

We now derive a general upper bound on the RC welfare loss, even when no uniform convergence is invoked.

\begin{theorem}[Welfare loss under threshold misspecification]
\label{thm:Wloss}
Suppose that $|\tau(X)| \le K$, $0 < \underline{c} \le c(X) \le \bar{c}$; the margin condition holds, that is, $P(|m^*(X)| \le t) \le C t$ for $t$ small. Then the welfare loss satisfies:
\[
\left|
\mathbb{E}[\tau(X)(\pi^*(X) - \pi^{RC}(X))]
\right|
\;\le\;
K C \cdot \mathbb{E}\big[c(X)\,|t^* - b_{LP}(X)|\big].
\]
\end{theorem}

\begin{proof}
We start from
\[
|\Delta|
=
\left|
\mathbb{E}[\tau(X)(\pi^*(X)-\pi^{RC}(X))]
\right|.
\]
Using $|\tau(X)| \le K$:
\[
|\Delta|
\le
K \cdot P(X \in \mathcal{M}_n).
\]
From Lemma 1
\[
\mathcal{M}_n \subseteq
\left\{
|m^*(X)| \le c(X)|t^* - b_{LP}(X)|
\right\}.
\]
Thus
\[
P(\mathcal{M}_n)
\le
P\left(
|m^*(X)| \le c(X)|t^* - b_{LP}(X)|
\right).
\]
Applying the margin condition
\[
P(|m^*(X)| \le u) \le C u,
\]
we obtain
\[
P(\mathcal{M}_n)
\le
C \cdot \mathbb{E}[c(X)|t^* - b_{LP}(X)|].
\]
Combining
\[
|\Delta|
\le
K C \cdot \mathbb{E}[c(X)|t^* - b_{LP}(X)|].
\]
\end{proof}

\bigskip
\noindent
This theorem confirms that the \textit{probability of misallocation} is controlled by the expected discrepancy between the constant threshold and the heterogeneous optimal threshold. \newline

\noindent
\textbf{Remark}. The previous theorem informally used the \textit{margin condition}, that is
\[
P(|m^*(X)|\le u)\le Cu,
\]
with a random radius \(u=c(X)|t^*-b_{LP}(X)|\). Since the standard margin condition is stated for deterministic \(u\), a more careful formulation is needed. We therefore impose the following strengthened version.

\medskip
\noindent
\textbf{Assumption} (\textit{functional margin condition}). There exists a constant \(C>0\) such that for every measurable function \(u:\mathcal X\to \mathbb R_+\), sufficiently small almost surely,
\[
P(|m^*(X)|\le u(X))\le C\,\mathbb E[u(X)].
\tag{FMC}
\]

\medskip
\noindent
This condition extends the standard margin condition from deterministic neighborhoods to random, covariate-dependent neighborhoods. We now show that the functional margin condition follows from a natural regularity assumption. Assume that the conditional distribution of $m^*(X)$ given $X=x$ admits a density $f_{m|X}(u|x)$, and that there exist constants $C>0$ and $\eta>0$ such that
\[
\sup_{|u|\le \eta,\; x \in \mathcal{X}} f_{m|X}(u|x) \le C.
\]
This means that the density of the margin is uniformly bounded in a neighborhood of zero. We now derive the functional margin condition step by step. By the law of total probability, we have
\[
P(|m^*(X)| \le u(X))
=
\mathbb{E}\left[
P(|m^*(X)| \le u(X) \mid X)
\right].
\]
Conditioning on $X=x$, the quantity $u(X)$ becomes deterministic
\[
P(|m^*(X)| \le u(X) \mid X=x)
=
\int_{-u(x)}^{u(x)} f_{m|X}(s|x)\,ds.
\]
Using the uniform bound on the density
\[
\int_{-u(x)}^{u(x)} f_{m|X}(s|x)\,ds
\le
\int_{-u(x)}^{u(x)} C\,ds
=
2C\,u(x).
\]
and taking expectation
\[
P(|m^*(X)| \le u(X))
\le
\mathbb{E}[2C\,u(X)]
=
2C\,\mathbb{E}[u(X)].
\]
showing that the functional margin condition holds (up to a multiplicative constant). The functional margin condition has a clear interpretation. The probability of being close to the decision boundary is proportional to the \emph{average width} of the neighborhood around that boundary. In other words: where the neighborhood is wide, it contributes more to the probability; where it is narrow, it contributes less. Overall, what matters is the expected width.

\medskip
\noindent
This condition implies that for each type of individual $x$, the fraction of units with $m^*(x)$ close to zero is limited, there is no excessive concentration of individuals who are nearly indifferent between treatment and no treatment. Thus, treatment decisions are typically well-separated and stable. Applying the functional margin condition to
\[
u(X) = c(X)\,|t^* - b_{LP}(X)|,
\]
we obtain
\[
P(\mathcal{M}_n)
\le
C\,\mathbb{E}[c(X)\,|t^* - b_{LP}(X)|].
\]
which shows that the probability of misallocation is controlled by the expected discrepancy between the constant threshold and the heterogeneous optimal threshold.

\section{RC and LP equivalence at constant threshold}

In this section, we show that when the LP threshold is approximately constant, as costs are constant over $X$, it is approximately equal to the constant threshold of the RC algorithm. As a consequence, the two assignment rules are approximately equivalent. Let us proceed as follows.

\bigskip
\noindent
Consider the previous constrained policy learning problem with constant costs:
\begin{equation}
\label{eq:max1}
\max_{\pi:\mathcal{X}\to\{0,1\}} \ \mathbb{E}[\tau(X)\pi(X)]
\end{equation}
subject to:
\begin{equation}
\underline{p} \le \mathbb{E}[\pi(X)] \le \overline{p},
\end{equation}
where $\underline{p}=\rho$, $\overline{p}=\frac{B}{c_0}$, and $c_0 > 0$ is a constant treatment cost. Fix the treatment share at a given value $m$ so that:
\begin{equation}
m := \mathbb{E}[\pi(X)] = \Pr(\pi(X)=1)
\end{equation}
We prove the following proposition:

\begin{proposition}[Threshold optimality under constant cost]
Fix any feasible treatment mass $m \in [\underline{p}, \overline{p}]$.  
An optimal solution to the previous constrained maximization is given by:
\begin{equation}
\pi_m(x) = \mathbf{1}\{\tau(x) \ge t(m)\}
\end{equation}
where $t(m)$ satisfies:
\begin{equation} 
\label{eq_sol1}
\Pr(\tau(X) \ge t(m)) = m
\end{equation}
If $\tau(X)$ has a continuous distribution, the threshold $t(m)$ is unique.
\end{proposition}

\begin{proof}
Fix $m$. The problem becomes:
\begin{equation}
\max_{\pi:\mathbb{E}[\pi]=m} \mathbb{E}[\tau(X)\pi(X)]
\end{equation}

\noindent
Let $F$ denote the distribution of $\tau(X)$. Then:
\begin{equation}
\mathbb{E}[\tau(X)\pi(X)] = \int \tau \cdot \pi(\tau)\, dF(\tau)
\end{equation}

\noindent
Consider any admissible policy $\pi$. If there exist $\tau_1 < \tau_2$ such that $\pi(\tau_1)=1$ and $\pi(\tau_2)=0$, swapping assignments increases the objective. Thus, optimal policies must assign treatment to the upper tail of $\tau(X)$. Therefore:
\begin{equation}
\pi_m(x) = \mathbf{1}\{\tau(x) \ge t(m)\}
\end{equation}
with $t(m)$ such that:
\begin{equation}
\Pr(\tau(X)\ge t(m)) = m
\end{equation}
Moreover, if $F$ is continuous, the threshold $t(m)$ is uniquely determined.
\end{proof}

\begin{corollary}[Uniqueness of the threshold]
Let $t_1$ and $t_2$ be two positive real numbers, with $t_1 < t_2$, then:
\begin{equation}
\Pr(\tau(X) > t_1) > \Pr(\tau(X) > t_2)
\end{equation}
Hence two distinct thresholds cannot generate the same mass $m$, implying uniqueness.
\end{corollary}

\noindent
The previous proposition and corollary set that, if any solution (\ref{eq_sol1}) to problem (\ref{eq:max1}) with fixed treatment mass $m$ admits a constant threshold, this threshold is unique. As both LP and RC comply with solution (\ref{eq_sol1}), when they admit a constant threshold, they are approximately equal. When is the LP rule admitting a constant threshold? It happens either when treatment costs are constant, or when they are heterogeneous, but $\nu$ is very close to zero. Therefore, four configurations can arise:

\begin{enumerate}
\item $c(X)$ heterogeneous, $\nu$ large $\Rightarrow$ high misallocation
\item $c(X)$ heterogeneous, $\nu$ small $\Rightarrow$ near equivalence
\item $c(X)$ constant, $\nu$ large $\Rightarrow$ exact equivalence
\item $c(X)$ constant, $\nu$ small $\Rightarrow$ exact equivalence
\end{enumerate}

\noindent
These prediction can be verified via Monte Carlo simulations across the four regimes, as proposed in section \ref{sec:mc2}.

\subsection{The \textit{rank-and-cut} implementation}

Motivated by the previous results,
we now introduce a particularly simple and computationally efficient
implementation of the constrained allocation problem. Rather than solving the 0--1 knapsack problem directly, or relying on
the Greedy--Lagrangian (GLC) procedure, we consider a rank-based rule
that orders units according to their cost-effectiveness ratio
\[
s_i := \frac{\hat\tau_i}{c_i}.
\]
This ratio can be interpreted as the estimated gain per unit of cost,
and provides a natural scalar index for prioritizing treatment.
The resulting procedure selects individuals in decreasing order of
$s_i$, and chooses the largest feasible prefix satisfying both the
budget constraint and the minimum coverage requirement.

\medskip
\noindent
The appeal of this approach lies in its simplicity: the algorithm
requires only a single sorting step followed by a linear scan over
feasible cutoffs, yielding an overall computational complexity of
$O(n \log n)$. In contrast to dual-based methods such as GLC, no iterative tuning of shadow prices is required. Importantly, despite its heuristic nature, the RC has a bounded welfare loss that is easily characterizable. The algorithm is described below.

\smallskip

{\small
\begin{algorithm}[H]
\caption{Cost-Effectiveness Ranking under Budget and Minimum Coverage}
\label{alg:opl_budget_rank}
\begin{enumerate}

\item \emph{Input.}
Estimated individual effects $(\hat\tau_i)_{i=1}^n$, individual costs $(c_i)_{i=1}^n$ with $c_i>0$,
total budget $C>0$, and minimum number treated $N^\ast\in\{1,\dots,n\}$.

\item \emph{Compute cost-effectiveness scores.}
For each $i$, define
\[
s_i := \frac{\hat\tau_i}{c_i}.
\]

\item \emph{Sort by decreasing score.}
Reorder units so that
\[
s_{(1)} \ge s_{(2)} \ge \cdots \ge s_{(n)}.
\]

\item \emph{Compute cumulative impact and cost.}
For each cutoff $k=1,\dots,n$, compute
\[
T_k := \sum_{j=1}^{k} \hat\tau_{(j)},
\qquad
C_k := \sum_{j=1}^{k} c_{(j)}.
\]

\item \emph{Identify feasible cutoffs.}
A cutoff $k$ is feasible if
\[
k \ge N^\ast
\quad\text{and}\quad
C_k \le C.
\]
If no feasible $k$ exists, stop (infeasible).

\item \emph{Choose optimal cutoff.}
Select
\[
k^\ast \in \arg\max\{T_k:\ k\ge N^\ast,\ C_k\le C\}.
\]

\item \emph{Construct policy.}
Define
\[
\pi_{(i)} := \mathbf{1}\{i \le k^\ast\},
\qquad i=1,\dots,n,
\]
and map back to the original order.

\item \emph{Output.}
Return $\pi$, total impact $\sum_i \hat\tau_i\pi_i$,
and total cost $\sum_i c_i\pi_i$.

\end{enumerate}
\end{algorithm}
}

\newpage

\section{Monte Carlo simulation 1: GLC vs. LP}

\subsection{Experimental setup}

We consider a simulation framework for evaluating the behavior of the GLC algorithm in a comparison with LP learning rule under budget and coverage constraints.
Within a controlled environment, where the true treatment effects and costs are known, the goal is to compare three allocation rules: (i) the exact optimal policy (OPT), which represents our oracle benchmark; (ii) the linear programming relaxation (LP); and (iii) the greedy--Lagrangian rule (GLC). 

\bigskip
\noindent
Our data generating process (DGP), considers a population of $n$ individuals indexed by $i = 1, \dots, n$, each characterized by a vector of observable features:
\[
X_i \in \mathbb{R}^d.
\]
In all simulations, covariates are generated as:
\[
X_i \sim \mathcal{N}(0, I_d),
\]
where $I_d$ denotes the $d \times d$ identity matrix, implying that covariates are independent across individuals. The \textit{individual treatment effect} is defined as a function of the covariates:
\[
\tau_i = \tau(X_i).
\]
In the baseline specification, we use a linear structure with only two covariates:
\[
\tau_i = X_{i1} + 0.5 X_{i2}.
\]
Each individual is associated with a strictly positive cost:
\[
w_i = w(X_i),
\]
which determines the budget required to treat that individual. To induce heterogeneity, we consider an exponential specification:
\[
w_i = \exp(\gamma X_{i1}),
\]
where $\gamma > 0$ controls the dispersion of costs. In our main simulations we set $\gamma = 2$, generating a high degree of cost heterogeneity. For completeness, we define potential outcomes as:
\[
Y_i(0) = \varepsilon_{i0}, \qquad
Y_i(1) = \tau_i + \varepsilon_{i1},
\]
where
\[
\varepsilon_{i0}, \varepsilon_{i1} \sim \mathcal{N}(0, \sigma^2)
\]
are independent noise terms. Since the true treatment effect $\tau_i$ is known, we evaluate policies using the \textit{oracle welfare}:
\[
V(\pi) = \sum_{i=1}^n \tau_i \pi_i.
\]

\bigskip
\noindent
The policy maker chooses a binary treatment rule:
\[
\pi_i \in \{0,1\},
\]
subject to two constraints: (i) budget constraint: $\sum_{i=1}^n w_i \pi_i \le W$; and coverage constraint: $\sum_{i=1}^n \pi_i \ge K$, where: $W = n \cdot C$, and $K = \lceil n \rho \rceil$. Here, $C$ represents the budget per capita and $\rho$ the minimum fraction of treated individuals. It goes without saying that the policy maker's objective is:
\[
\max_{\pi \in \{0,1\}^n}
\sum_{i=1}^n \tau_i \pi_i
\quad
\text{s.t.}
\quad
\sum_{i=1}^n w_i \pi_i \le W,
\quad
\sum_{i=1}^n \pi_i \ge K.
\]

\bigskip
\noindent
We evaluate four different \textit{allocation rules}:

\bigskip
\noindent
\textit{OPT (exact solution).}
The exact solution is obtained by solving the \textit{combinatorial} optimization problem using mixed-integer linear programming (MILP). This provides the benchmark optimal allocation of our experiment.

\bigskip
\noindent
\textit{LP (relaxation).} The binary constraint $\pi_i \in \{0,1\}$ is relaxed to $\pi_i \in [0,1]$, yielding a linear program. This provides an upper bound on the optimal value.

\bigskip
\noindent
\textit{GLC (greedy--Lagrangian).}
The GLC rule assigns a score, $a_i(\lambda) = \tau_i - \lambda w_i$, where $\lambda$ is a Lagrange multiplier. Individuals are ranked according to this score, and the allocation is constructed by: (i) first enforcing the coverage constraint; and (ii) then adding units with positive adjusted gain while respecting the budget. The multiplier $\lambda$ is selected via \textit{bisection} to satisfy the budget constraint.

\bigskip
\noindent
In order to evaluate each method, we use the following performance metrics:

\begin{itemize}
\item \textit{Total welfare:} $\sum_i \tau_i \pi_i$,
\item \textit{Per-capita welfare:} $\frac{1}{n} \sum_i \tau_i \pi_i$,
\item \textit{Regret:}
\[
\text{Regret} = V(\pi^{OPT}) - V(\pi),
\]
\item \textit{Per-capita regret:}
\[
\frac{1}{n} \left( V(\pi^{OPT}) - V(\pi) \right),
\]
\item \textit{Integrality gap (LP only):}
\[
V(\pi^{LP}) - V(\pi^{OPT}).
\]
\end{itemize}

\bigskip
\noindent
We devise a Monte Carlo experiment, repeating the analysis across different sample sizes from 50 to 500 by step of 25. For each value of $n$, we generate $R = 50$ independent datasets. For each replication, we proceed as follows: (i) we generate $(X_i, \tau_i, w_i)$; (ii) we solve the allocation problem using all methods; (iii) we compute welfare and regret; (iv) we aggregate results across replications. The simulation evaluates how close the GLC policy rules is to the optimal solution, and how its performance evolves as the sample size increases.

\subsection{Experimental results}

This section presents the results of the Monte Carlo experiment described above. Table~\ref{tab:mc_main_results} reports the average results of our Monte Carlo, while Figure~\ref{fig:regret_vs_n} complements the table by illustrating the behavior of the LP and GLC regrets across increasing values of $n$.

\begin{table}[htbp]
\centering
\caption{Monte Carlo comparison of OPT, LP and GLC}
\label{tab:mc_main_results}
\begin{tabular}{rccccc}
\toprule
& \multicolumn{2}{c}{Per-capita value} & \multicolumn{1}{c}{GLC} & \multicolumn{2}{c}{LP} \\
\cmidrule(lr){2-3}\cmidrule(lr){4-4}\cmidrule(lr){5-6}
$n$ & OPT & GLC & regret & gap & frac. comp. \\
\midrule
 50  & 0.2001 & 0.1954 & 0.0047 & 0.0044 & 1.6 \\
100  & 0.2045 & 0.1986 & 0.0059 & 0.0010 & 1.3 \\
150  & 0.2295 & 0.2272 & 0.0023 & 0.0006 & 1.3 \\
200  & 0.2120 & 0.2109 & 0.0010 & 0.0004 & 1.3 \\
250  & 0.2113 & 0.2098 & 0.0015 & 0.0002 & 1.0 \\
300  & 0.2138 & 0.2120 & 0.0017 & 0.0002 & 1.2 \\
350  & 0.2091 & 0.2081 & 0.0010 & 0.0001 & 1.3 \\
400  & 0.2070 & 0.2056 & 0.0014 & 0.0001 & 1.6 \\
450  & 0.2061 & 0.2051 & 0.0011 & 0.0001 & 1.4 \\
500  & 0.2112 & 0.2099 & 0.0013 & 0.0001 & 1.2 \\
\bottomrule
\end{tabular}
\begin{flushleft}
\footnotesize
Notes: OPT value = mean per-capita objective value under the exact MILP solution. 
GLC value = mean per-capita objective value under the Greedy--Lagrangian with coverage rule. 
GLC regret = mean per-capita regret relative to OPT. 
LP value = mean per-capita objective value under the LP relaxation. 
LP gap = mean per-capita integrality gap, i.e.\ LP minus OPT. 
LP frac. = mean number of fractional LP components.
\end{flushleft}
\end{table}

\noindent
Table \ref{tab:mc_main_results} highlights several important patterns regarding the relative performance of OPT, LP, and GLC across different sample sizes.

\bigskip
\noindent
First, as expected, OPT consistently delivers the highest per-capita value, serving as the benchmark for all comparisons. The GLC rule performs remarkably well, with a per-capita value that is very close to OPT even for relatively small sample sizes. The associated regret is modest already at $n=50$ (around $0.0047$) and decreases rapidly as $n$ increases, stabilizing around values close to $0.001$ for $n \geq 200$. This provides strong empirical evidence that GLC is an accurate approximation of the optimal allocation, even in finite samples.

\bigskip
\noindent
Second, the LP relaxation exhibits a very small integrality gap throughout, which shrinks quickly with $n$. Starting from $0.0044$ at $n=50$, the gap drops to negligible levels (around $0.0001$--$0.0002$) for $n \geq 250$. This confirms that the LP solution is an excellent proxy for the combinatorial optimum in per-capita terms, supporting the theoretical argument of asymptotic equivalence between LP and OPT.

\bigskip
\noindent
Third, the number of fractional components in the LP solution remains consistently low, fluctuating around $1$--$2$ across all sample sizes. This is fully consistent with the known structure of LP extreme points under a small number of global constraints (budget and coverage), and further explains why the integrality gap is negligible in practice.

\begin{figure}[!h]
\centering
\includegraphics[width=0.7\textwidth]{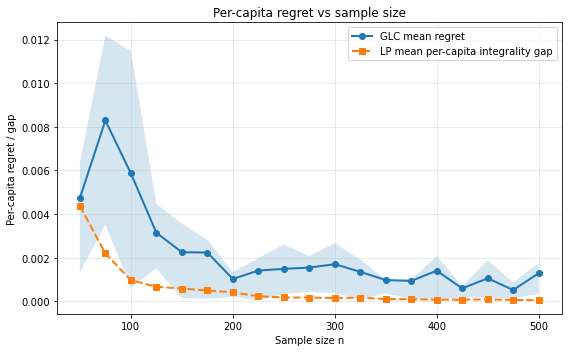}
\caption{Per-capita regret of GLC and integrality gap of LP as a function of sample size. The solid line reports the mean regret of GLC relative to OPT, while the dashed line reports the mean per-capita integrality gap of the LP relaxation. Shaded areas represent variability across Monte Carlo replications (interquantile range).}
\label{fig:regret_vs_n}
\end{figure}

\bigskip
\noindent
Figure \ref{fig:regret_vs_n} provides a graphical counterpart to the Monte Carlo results. Here, we can  graphically appreciate how the per-capita regret of GLC decreases rapidly with the sample size and stabilizes at very low levels. Moreover, the LP integrality gap is uniformly small and quickly converges to zero, becoming negligible already for moderate sample sizes.

\section{Monte Carlo simulation 2: LP and RC equivalence} \label{sec:mc2}

\subsection{Data Generating Process and Simulation Design}
Let $A \in \{0,1\}$ be a binary treatment variable, $X \sim \mathcal{N}(0,1)$ a continuous confounder, and define a treatment framework with potential outcomes 
\[
Y(0) = \beta_0 X + \varepsilon_0, \quad \varepsilon_0 \sim \mathcal{N}(0,1),
\]
\[
Y(1) = \beta_1 X + \gamma X^2 + \varepsilon_1, \quad \varepsilon_1 \sim \mathcal{N}(0,1).
\]
and observed outcome
\[
Y = Y(0) + A \cdot (Y(1) - Y(0)).
\]
The Conditional Average Treatment Effect (CATE) is thus defined as
\[
\tau(X) = (\beta_1 - \beta_0) X + \gamma X^2,
\]
which introduces nonlinear heterogeneity across individuals.
For each individual characterized by $X$, we define an individual treatment cost as
\[
c(X) = c_0 + \delta |X|,
\]
where $\delta \geq 0$ controls cost heterogeneity.
We impose the following budgetary and coverage constrains
\[
\frac{1}{n} \sum_{i=1}^n c_i \pi_i \leq B, \quad
\frac{1}{n} \sum_{i=1}^n \pi_i \geq \rho.
\] 

\bigskip
\noindent
In this simulation design, we fix a sample size $n = 500$ and a number of replication $R = 100$. For each replication, we compute: (i) the LP allocation $\pi^{LP}$ (solution of the relaxed problem); (ii) the RC allocation $\pi^{RC}$ (the rank-and-cut rule calibrated to match LP coverage); (iii) the misallocation area:
$
\mathcal{M} = \frac{1}{n} \sum_{i=1}^n \mathbf{1}\{\pi_i^{LP} \neq \pi_i^{RC}\}.
$
Results are averaged over the $R$ replications. We consider the following four simulation scenarios:
\begin{enumerate}
\item High cost heterogeneity, high coverage: $\delta = \delta_H$, $\rho = \rho_H$.
\item High cost heterogeneity, low coverage: $\delta = \delta_H$, $\rho = \rho_L$.
\item Constant cost, high coverage: $\delta = 0$, $\rho = \rho_H$.
\item Constant cost, low coverage: $\delta = 0$, $\rho = \rho_L$.
\end{enumerate}
In the simulations, we set
$\delta_H = 1$, $\rho_H = 0.5$, and $\rho_L = 0.1$.
Table \ref{tab:mcLPRC} show simulation results.

\bigskip
\noindent
\begin{table}[H]
\centering
\caption{Misallocation area between LP and RC rules}
\label{tab:mcLPRC}
\begin{tabular}{lccccc}
\hline
Scenario & Cost Heterogeneity & Coverage ($\rho$) & Mean $\nu$ & Status & Misallocation Area \\
\hline
(1) High $\delta$, High $\rho$ & High & 0.5 & 0.842 & Binding & 0.1520 \\
(2) High $\delta$, Low $\rho$  & High & 0.1 & 0.000 & Slack   & 0.0010 \\
(3) $\delta = 0$, High $\rho$  & None & 0.5 & 0.000 & Slack   & 0.0020 \\
(4) $\delta = 0$, Low $\rho$   & None & 0.1 & 0.000 & Slack   & 0.0020 \\
\hline
\end{tabular}
\end{table}

\noindent
The Monte Carlo evidence strongly supports the theoretical structure of the LP solution $\pi^{LP}(x) = 1\{\tau(x) - \lambda^* c(x) + \nu^* \geq 0\}$.
The key driver of misallocation is the interaction between:
(i) cost heterogeneity $c(x)$, and (ii) the coverage shadow price $\nu^*$.

\bigskip
\noindent
Scenario (1) generates large misallocation. Indeed, when $\delta$ is large, costs vary substantially across individuals. When $\rho$ is also large, the coverage constraint is binding, implying $\nu^* > 0$. The LP rule becomes
$ \tau(x) - \lambda^* c(x) + \nu^* \geq 0$, which defines an \emph{affine decision boundary} in $(\tau,c)$ space. In contrast, RC ranks units by
$\frac{\tau(x)}{c(x)}$, which corresponds to a \emph{ratio-based ordering}.
Affine and ratio thresholds are fundamentally different objects. When both $c(x)$ is heterogeneous and $\nu^*$ is non-negligible, these two orderings diverge substantially, generating a large misallocation region. This is exactly what we observe in Table 1 (with misallocation area $\approx 0.152$).

\bigskip
\noindent
Scenario (2) yields negligible misallocation. When $\rho$ is small, the coverage constraint is slack, implying $\nu^* \approx 0$. The LP rule reduces to:
\[
\tau(x) - \lambda^* c(x) \geq 0 \quad \Leftrightarrow \quad \frac{\tau(x)}{c(x)} \geq \lambda^*,
\]
which is a ratio rule identical to RC. Hence, the two policies approximately coincide, and the misallocation area collapses to $\approx 0.001$.

\bigskip
\noindent
Scenarios (3) and (4) yield negligible misallocation as well. When costs are constant ($c(x)=c$), the LP rule becomes:
\[
\tau(x) - \lambda^* c + \nu^* \equiv \tau(x) + \text{constant}.
\]
Thus, the LP policy is equivalent to ranking units by $\tau(x)$.
Similarly, the RC rule reduces to:
\[
\frac{\tau(x)}{c} \propto \tau(x),
\]
which induces the same ordering. Therefore, regardless of whether $\nu^*$ is large or small, the two rules coincide up to negligible differences, as confirmed in Table 1.

\bigskip
\noindent
To conclude, severe misallocation arises only when: (i) cost heterogeneity is present ($\delta > 0$), and (ii) the coverage constraint is binding ($\nu^* \gg 0$). In all other cases, LP and RC are approximately equivalent.\footnote{
By complementary slackness in the Karush–Kuhn–Tucker (KKT) conditions, the dual variable associated with the coverage constraint satisfies
$
\nu^* \cdot \left( \mathbb{E}[\pi^*(X)] - \rho \right) = 0.
$
This implies that if $\nu^* > 0$, then necessarily $\mathbb{E}[\pi^*(X)] = \rho$, i.e., the coverage constraint is binding. Conversely, if $\mathbb{E}[\pi^*(X)] > \rho$, then $\nu^* = 0$, meaning that the constraint is slack and does not affect the optimal allocation. Importantly, a large value of $\rho$ does not, per se, imply that $\nu^* > 0$. Having a positive $\nu^*$ depends on whether the decision-maker find it optimal to treat the minimum selected percentage (binding situation), which is likelier, but not granted, when $\rho$ is high. As a counter example, scenario (3) shows a situation where $\rho$ is high, but $\nu=0$.}

\section{Conclusion}

This paper studies optimal policy learning in the presence of a combined budget and minimum coverage constraints. This setting naturally arises in many real-world allocation problems. We show that the problem admits a knapsack structure and derive a primal--dual characterization of the optimal policy in terms of an affine threshold rule driven by two shadow prices: the budget multiplier and the coverage multiplier.

\bigskip
\noindent
Our analysis highlights a fundamental structural insight. While the budget constraint alone leads to a ratio-based allocation rule, the introduction of a minimum coverage requirement fundamentally alters the geometry of the problem. In particular, when the coverage constraint is binding, the optimal policy depends on an additive shift induced by the coverage shadow price, generating an affine decision boundary that cannot, in general, be replicated by simple ranking rules.

\bigskip
\noindent
We show that the linear programming relaxation of the empirical problem has an $O(1)$ integrality gap, implying that the per-capita difference between the relaxed and discrete solutions vanishes asymptotically. This result provides a rigorous justification for using tractable approximation algorithms in large-scale applications.

\bigskip
\noindent
We then analyze two implementable policies. The Greedy--Lagrangian (GLC) algorithm exploits the dual structure of the problem and achieves asymptotic optimality through a one-dimensional search over the budget multiplier. The rank-and-cut (RC) rule, by contrast, offers a simpler and non-iterative alternative based on cost-effectiveness ranking.

\bigskip
\noindent
A central contribution of the paper is to characterize precisely when this simplification is approximately valid. In particular, both theory and Monte Carlo evidence show that the discrepancies between the optimal policy and the RC rule is likely to arise mainly when two conditions jointly hold: (i) costs are heterogeneous, and (ii) the coverage constraint is binding, implying a strictly positive coverage shadow price. In all other cases, the two rules are approximately equivalent.

\bigskip
\noindent
These findings provide a clear practical message. In many empirically relevant settings, simple ranking-based rules can achieve near-optimal performance at a negligible computational cost. At the same time, the analysis clarifies the limits of such rules and identifies the structural conditions under which more refined algorithms are required.

\bigskip
\noindent
More broadly, the paper contributes to bridging the gap between statistical policy learning and combinatorial optimization, offering a unified framework that is both theoretically grounded and computationally scalable.

\bibliographystyle{apalike}
\bibliography{References.bib}

\newpage

\footnotesize{
\section*{\footnotesize{Appendix 1. A numerical illustration of the GLC algorithm}}
\label{app:glc_example}

This appendix provides a fully worked numerical example of Algorithm~1, namely the Greedy--Lagrangian with minimum coverage and budget (GLC). The purpose is purely illustrative. In particular, the example is constructed so as to highlight a regime shift in the selected policy as the Lagrange multiplier changes. This helps clarify the role of the bisection step and the non-smooth dependence of the selected set on the budget shadow price.

\medskip
\noindent
\textit{Setup of the example}. Consider a sample of $n=6$ units. For each unit $i$, let $v_i$ denote its value and $w_i$ its cost. We take $v = (20,18,14,13,8,7)$, and $w = (10,9,4,4,2,2)$. The total budget is $W_n = 12$, the minimum coverage requirement is $K_n = 2$, and the stopping tolerance is $\varepsilon = 0.05$, with $\varepsilon W_n = 0.6$. The six units are summarized in the following table:

\begin{table}[htbp]
\centering
\footnotesize
\begin{tabular}{c c c}
\hline
Unit $i$ & Value $v_i$ & Cost $w_i$ \\
\hline
1 & 20 & 10 \\
2 & 18 & 9 \\
3 & 14 & 4 \\
4 & 13 & 4 \\
5 & 8  & 2 \\
6 & 7  & 2 \\
\hline
\end{tabular}
\end{table}

\noindent
This example is intentionally designed to feature two qualitatively different groups of units. Units 1 and 2 have very large values, but they are also very expensive. By contrast, units from 3 to 6 have smaller absolute values, but much better value-to-cost trade-offs. As a consequence, for small values of the multiplier $\lambda$, the ranking is dominated by raw values and the expensive units tend to be selected first. For larger values of $\lambda$, costs are penalized more strongly, and the ranking shifts toward more efficient units.

\medskip
\noindent
\textit{Preliminary feasibility check}.
The first step of the algorithm is to verify whether the minimum coverage requirement can be met within the available budget. To do so, one sorts costs in increasing order and sums the $K_n$ smallest ones, thus obtaining $(2,2,4,4,9,10)$. Since $K_n=2$, the minimum feasible coverage cost is $W_{\min K} = 2+2 = 4$. Because $W_{\min K} = 4 \leq 12 = W_n$, the problem is feasible. This check is important because it isolates a purely combinatorial impossibility. If even the two cheapest units could not be jointly afforded, then no policy satisfying the coverage constraint would exist, regardless of the choice of $\lambda$.

\medskip
\noindent
\textit{Interpretation of the bisection step}. Before entering the numerical iterations, it is useful to explain the logic of bisection. The algorithm defines, for a given multiplier $\lambda$, the score
$a_i(\lambda) = v_i - \lambda w_i$.
These scores induce an ordering of the units. The first $K_n$ units in that ranking form the forced core required by the coverage constraint. The algorithm then attempts to augment this core greedily, adding further positively scored units as long as the budget allows. The multiplier $\lambda$ acts as a shadow price of the budget. When $\lambda$ is low, cost is weakly penalized, so expensive high-value units tend to dominate the ranking. When $\lambda$ is high, cost is heavily penalized, and the ranking tilts toward more cost-effective units. Bisection is simply a search procedure over $\lambda$. One starts from an interval $[\lambda_L,\lambda_U]$ known to bracket the relevant solution. At each iteration one computes the midpoint $\lambda_M = \frac{\lambda_L+\lambda_U}{2}$, runs the selection rule at $\lambda_M$, and checks whether the resulting policy is too expensive or leaves too much budget unused. If the selected set is too expensive, then $\lambda_M$ is too small, so the lower bound is increased. If the selected set is too cheap, then $\lambda_M$ is too large, so the upper bound is decreased. Repeating this operation shrinks the interval until the selected budget is sufficiently close to $W_n$. In this sense, bisection is not optimizing directly over subsets. Rather, it is searching for the multiplier that induces a subset with nearly binding budget.

\medskip
\noindent
We are now ready to run iterations. 

\bigskip
\noindent
\textit{Iteration 1: low multiplier and infeasible core}. We first consider a low multiplier, $\lambda = 0.2$.
The corresponding scores are: $a_i = v_i - 0.2\,w_i$. The following table reports the calculations:

\begin{table}[htbp]
\centering
\footnotesize
\begin{tabular}{c c c c}
\hline
Unit $i$ & $v_i$ & $w_i$ & $a_i = v_i - 0.2w_i$ \\
\hline
1 & 20 & 10 & 18.0 \\
2 & 18 & 9  & 16.2 \\
3 & 14 & 4  & 13.2 \\
4 & 13 & 4  & 12.2 \\
5 & 8  & 2  & 7.6 \\
6 & 7  & 2  & 6.6 \\
\hline
\end{tabular}
\end{table}

\noindent
Ranking the units by decreasing score yields: $1 \succ 2 \succ 3 \succ 4 \succ 5 \succ 6$. Since $K_n=2$, the coverage step forces the first two units in this ranking into the selected set:
$
S = \{1,2\}.
$
The corresponding cost and value are
$
B = w_1+w_2 = 10+9 = 19,
$
$
V = v_1+v_2 = 20+18 = 38.
$
This is already infeasible, because
$
B = 19 > 12 = W_n.
$
The result is summarized in the following table:
\begin{table}[h!]
\centering
\footnotesize
\begin{tabular}{l c}
\hline
Selected core $S$ & $\{1,2\}$ \\
Coverage size & $2$ \\
Total cost $B$ & $19$ \\
Total value $V$ & $38$ \\
Feasible? & No \\
\hline
\end{tabular}
\end{table}

\noindent
At this stage the algorithm does not even reach the greedy augmentation step. The reason is structural: the multiplier is too small, so the ranking is too close to a pure ordering by values, and the coverage requirement immediately forces the inclusion of two very expensive units. In terms of the bisection logic, this means that $\lambda$ is too low and should be increased.

\bigskip
\noindent
\textit{Iteration 2: higher multiplier and a regime shift in the ranking}. Now consider a much larger multiplier,
$
\lambda = 1.2.
$
The scores become
$
a_i = v_i - 1.2\,w_i.
$
The following table gives the corresponding values:
\begin{table}[h!]
\centering
\footnotesize
\begin{tabular}{c c c c}
\hline
Unit $i$ & $v_i$ & $w_i$ & $a_i = v_i - 1.2w_i$ \\
\hline
1 & 20 & 10 & 8.0 \\
2 & 18 & 9  & 7.2 \\
3 & 14 & 4  & 9.2 \\
4 & 13 & 4  & 8.2 \\
5 & 8  & 2  & 5.6 \\
6 & 7  & 2  & 4.6 \\
\hline
\end{tabular}
\end{table}

\noindent
Now the ranking changes to
$
3 \succ 4 \succ 1 \succ 2 \succ 5 \succ 6.
$
This is the central feature of the example. The expensive units 1 and 2, which previously occupied the top of the ranking, are now overtaken by units 3 and 4. The reason is that the larger multiplier penalizes cost more heavily, so the ranking is no longer driven by raw values alone, but rather by net value after cost penalization.
The coverage step now selects
$
S = \{3,4\},
$
with
$
B = w_3+w_4 = 4+4 = 8$, and
$V = v_3+v_4 = 14+13 = 27$.
This core is feasible. The following table summarizes the result:

\begin{table}[h!]
\centering
\footnotesize
\begin{tabular}{l c}
\hline
Selected core $S$ & $\{3,4\}$ \\
Coverage size & $2$ \\
Total cost $B$ & $8$ \\
Total value $V$ & $27$ \\
Feasible? & Yes \\
\hline
\end{tabular}
\end{table}

\bigskip
\noindent
\textit{Greedy augmentation at $\lambda = 1.2$}
Since the core is feasible, the algorithm proceeds to the greedy augmentation step. The remaining units are scanned in the score order:
$
1, 2, 5, 6.
$
We evaluate them one by one:
\begin{itemize}

\item Unit 1.
Adding unit 1 would produce
$
B + w_1 = 8 + 10 = 18 > 12,
$
so unit 1 cannot be added.

\item Unit 2.
Adding unit 2 would produce
$
B + w_2 = 8 + 9 = 17 > 12,
$
so unit 2 cannot be added either.

\item Unit 5.
Adding unit 5 gives
$
B + w_5 = 8 + 2 = 10 \leq 12,
$
so unit 5 is admitted. The updated totals are
$
S = \{3,4,5\}$,
$
B = 10
$, and
$V = 35.
$

\item Unit 6.
Adding unit 6 then gives
$
B + w_6 = 10 + 2 = 12 \leq 12,
$
so unit 6 is also admitted. The final selected set becomes
$
S = \{3,4,5,6\},
$
and 
$
B = 12$, and 
$
V = 42.
$ 
\end{itemize}

\noindent
These calculations are summarized in the following table:

\begin{table}[h!]
\centering
\footnotesize
\begin{tabular}{c c c c c}
\hline
Candidate unit & Score rank position & Cost if added & Feasible? & Action \\
\hline
1 & 3rd overall & $8+10=18$ & No  & Reject \\
2 & 4th overall & $8+9=17$  & No  & Reject \\
5 & 5th overall & $8+2=10$  & Yes & Add \\
6 & 6th overall & $10+2=12$ & Yes & Add \\
\hline
\end{tabular}
\end{table}

\noindent
At the end of this step the slack is
$
W_n - B = 12-12 = 0.
$
Since
$
0 \leq \varepsilon W_n = 0.6,
$
the stopping condition is satisfied. Therefore the algorithm terminates with the policy
$
\pi = (0,0,1,1,1,1).
$

\bigskip
\noindent
\textit{Final solution}. The final policy selects units 3, 4, 5, and 6. Its total value and total cost are
$
V = 14+13+8+7 = 42
$ and 
$B = 4+4+2+2 = 12.
$
Thus the budget is exactly exhausted and the coverage requirement is more than satisfied. The following table compares the low-$\lambda$ and high-$\lambda$ regimes:

\begin{table}[h!]
\centering
\footnotesize
\begin{tabular}{c c c c c}
\hline
Multiplier $\lambda$ & Ranking head & Forced core & Total cost of core & Outcome \\
\hline
0.2 & $(1,2,\ldots)$ & $\{1,2\}$ & $19$ & Infeasible \\
1.2 & $(3,4,\ldots)$ & $\{3,4\}$ & $8$  & Feasible, then augmented to $\{3,4,5,6\}$ \\
\hline
\end{tabular}
\end{table}

\bigskip
\noindent
By and large, this example illustrates several important features of the GLC algorithm.

\bigskip
\noindent
First, the minimum coverage requirement is not a minor side condition. It fundamentally shapes the selected policy, because it forces the algorithm to include the top $K_n$ units in the score ranking. If the multiplier is too small, those top-ranked units may be very expensive, making the forced core infeasible even before any greedy augmentation takes place. 

\bigskip
\noindent
Second, the multiplier $\lambda$ has a structural rather than merely marginal role. Increasing $\lambda$ does not simply trim the selected set at the margin. It may reorder the ranking itself, thereby changing the identity of the forced core. In the present example, the core changes from $\{1,2\}$ to $\{3,4\}$ once cost penalization becomes sufficiently strong.

\bigskip
\noindent
Third, the mapping from $\lambda$ to the selected policy is inherently non-smooth. Small changes in $\lambda$ may leave the ranking unchanged for a while, but once a score crossing occurs, the selected set may jump discretely. This is precisely what makes the problem combinatorial. The selected policy is piecewise constant in $\lambda$, with abrupt changes at threshold values where the ranking is reordered.

\bigskip
\noindent
Finally, the example also clarifies the logic of the bisection step. The algorithm is not searching directly over all feasible subsets. Instead, it searches over the one-dimensional dual parameter $\lambda$, which in turn induces a candidate subset through ranking, forced coverage, and greedy augmentation. In this sense, bisection provides a practical way to navigate a difficult discrete problem through a scalar dual variable.}

\footnotesize{
\section*{\footnotesize{Appendix 2. A numerical illustration of the rank-and-cut (RC) algorithm}}
\label{app:rc_example}

This appendix provides a fully worked numerical example of the rank-and-cut (RC) algorithm under budget and minimum coverage constraints. The purpose is purely illustrative. In particular, the example is constructed to highlight how the ratio-based ranking interacts with feasibility constraints and how the selected set emerges from a sequential selection procedure.

\medskip
\noindent
\textit{Setup of the example}. Consider the same sample of $n=6$ units used in Appendix~1. For each unit $i$, let $v_i$ denote its value and $w_i$ its cost. We take
$v = (20,18,14,13,8,7)$ and $w = (10,9,4,4,2,2)$.
The total budget is $W_n = 12$ and the minimum coverage requirement is $K_n = 2$. The data are summarized below:

\begin{table}[htbp]
\centering
\footnotesize
\begin{tabular}{c c c}
\hline
Unit $i$ & Value $v_i$ & Cost $w_i$ \\
\hline
1 & 20 & 10 \\
2 & 18 & 9 \\
3 & 14 & 4 \\
4 & 13 & 4 \\
5 & 8  & 2 \\
6 & 7  & 2 \\
\hline
\end{tabular}
\end{table}

\medskip
\noindent
\textit{Ratio-based ranking}. The RC algorithm ranks units according to their value-to-cost ratio:
\[
r_i = \frac{v_i}{w_i}.
\]
The computed ratios are:

\begin{table}[htbp]
\centering
\footnotesize
\begin{tabular}{c c c c}
\hline
Unit $i$ & $v_i$ & $w_i$ & $r_i = v_i / w_i$ \\
\hline
1 & 20 & 10 & 2.00 \\
2 & 18 & 9  & 2.00 \\
3 & 14 & 4  & 3.50 \\
4 & 13 & 4  & 3.25 \\
5 & 8  & 2  & 4.00 \\
6 & 7  & 2  & 3.50 \\
\hline
\end{tabular}
\end{table}

\noindent
Sorting in decreasing order yields:
\[
5 \succ 3 \succ 6 \succ 4 \succ 1 \succ 2.
\]

\medskip
\noindent
\textit{Step 1: Coverage enforcement}. The RC algorithm first ensures that the minimum coverage requirement is satisfied. Therefore, it selects the first $K_n=2$ units in the ranking:
\[
S = \{5,3\}.
\]
The corresponding totals are:
\[
B = w_5 + w_3 = 2 + 4 = 6, \quad V = v_5 + v_3 = 8 + 14 = 22.
\]

\noindent
This core is feasible, since $B = 6 \leq 12$.

\medskip
\noindent
\textit{Step 2: Greedy augmentation}. The algorithm proceeds by scanning the remaining units in ratio order:
\[
6, 4, 1, 2.
\]

\begin{itemize}

\item Unit 6.  
Adding unit 6 gives:
\[
B + w_6 = 6 + 2 = 8 \leq 12,
\]
so unit 6 is added:
\[
S = \{5,3,6\}, \quad B = 8, \quad V = 29.
\]

\item Unit 4.  
Adding unit 4 gives:
\[
B + w_4 = 8 + 4 = 12 \leq 12,
\]
so unit 4 is added:
\[
S = \{5,3,6,4\}, \quad B = 12, \quad V = 42.
\]

\item Unit 1.  
Adding unit 1 would give:
\[
12 + 10 = 22 > 12,
\]
so it is rejected.

\item Unit 2.  
Adding unit 2 would give:
\[
12 + 9 = 21 > 12,
\]
so it is rejected.

\end{itemize}

\medskip
\noindent
These steps are summarized below:

\begin{table}[h!]
\centering
\footnotesize
\begin{tabular}{c c c c c}
\hline
Candidate unit & Ratio rank position & Cost if added & Feasible? & Action \\
\hline
6 & 3rd & $6+2=8$  & Yes & Add \\
4 & 4th & $8+4=12$ & Yes & Add \\
1 & 5th & $12+10=22$ & No & Reject \\
2 & 6th & $12+9=21$  & No & Reject \\
\hline
\end{tabular}
\end{table}

\medskip
\noindent
\textit{Final solution}. The final selected set is:
\[
S = \{5,3,6,4\},
\]
with total value and cost:
\[
V = 42, \quad B = 12.
\]

\noindent
The budget is exactly exhausted and the coverage requirement is satisfied.

\bigskip
\noindent
\textit{Interpretation}. This example highlights the structural features of the RC algorithm: (i) the ranking is entirely driven by the ratio $v_i/w_i$. Unlike the GLC algorithm, there is no dual parameter that adjusts the trade-off between value and cost. As a consequence, the ordering of units is fixed and does not depend on any tuning parameter; (ii) the minimum coverage requirement enters only through the selection of the first $K_n$ units. Once this core is formed, the algorithm proceeds greedily without revisiting earlier decisions; (iii) the RC algorithm produces a single deterministic allocation given the data. In contrast with GLC, there is no regime shift induced by a changing multiplier. The absence of a dual parameter implies that the mapping from data to policy is smooth in a trivial sense, but potentially suboptimal when compared to LP-based rules.}

\end{document}